\documentclass{article}
\usepackage[final,main]{neurips_2026}

\usepackage[utf8]{inputenc}
\usepackage[T1]{fontenc}
\usepackage{amsmath,amssymb,amsfonts}
\usepackage{amsthm}
\usepackage{booktabs}
\usepackage{graphicx}
\usepackage{xcolor}
\usepackage{algorithm}
\usepackage{algorithmic}
\usepackage{hyperref}
\usepackage{cleveref}
\usepackage{microtype}
\usepackage{multirow}
\usepackage{subcaption}
\usepackage{enumitem}
\usepackage{wrapfig}

\newtheorem{theorem}{Theorem}[section]
\newtheorem{lemma}[theorem]{Lemma}
\newtheorem{proposition}[theorem]{Proposition}
\newtheorem{corollary}[theorem]{Corollary}
\newtheorem{definition}{Definition}

\newtheorem{remark}{Remark}

\crefname{theorem}{Theorem}{Theorems}
\crefname{lemma}{Lemma}{Lemmas}
\crefname{proposition}{Proposition}{Propositions}
\crefname{corollary}{Corollary}{Corollaries}
\crefname{definition}{Definition}{Definitions}
\crefname{assumption}{Assumption}{Assumptions}
\crefname{remark}{Remark}{Remarks}
\Crefname{theorem}{Theorem}{Theorems}
\Crefname{lemma}{Lemma}{Lemmas}
\Crefname{proposition}{Proposition}{Propositions}
\Crefname{corollary}{Corollary}{Corollaries}
\Crefname{definition}{Definition}{Definitions}
\Crefname{assumption}{Assumption}{Assumptions}
\Crefname{remark}{Remark}{Remarks}

\newcommand{\R}{\mathbb{R}}
\newcommand{\E}{\mathbb{E}}
\DeclareMathOperator*{\argmin}{arg\,min}

\DeclareMathOperator{\tr}{tr}
\DeclareMathOperator{\rank}{rank}
\DeclareMathOperator{\diag}{diag}

\newcommand{\nT}{T}                    
\newcommand{\nL}{L}                    
\newcommand{\nK}{K}                    
\newcommand{\dout}{d_{\mathrm{out}}}
\newcommand{\din}{d_{\mathrm{in}}}

\newcommand{\dW}{\Delta W}             
\newcommand{\dWt}{\widetilde{\Delta W}} 

\newcommand{\fnorm}[1]{\left\|#1\right\|_{\mathrm{F}}}

\title{GRAM: Group-wise Rank-Aware Model Merging via Subspace Alignment}

\renewcommand*{\thefootnote}{\fnsymbol{footnote}}
\author{%
  Di Wu\thanks{Equal contribution.} \\
  JPMorgan Chase \& Co. \\
  \texttt{di.x5.wu@chase.com} \\
  \And
  Shuaidong Pan\footnotemark[1] \\
  JPMorgan Chase \& Co. \\
  \texttt{shuaidong.pan@chase.com} \\
}

\begin{document}

\maketitle
\renewcommand*{\thefootnote}{\arabic{footnote}}
\setcounter{footnote}{0}

\begin{abstract}
Merging different models into a single model is highly desirable for multi-task deployment,
particularly in regulated domains like healthcare. In such settings, models are often locally fine-tuned due to strict data privacy policies, necessitating that aggregation operates exclusively on the fine-tuned weights.  Existing methods apply a uniform merge across all
tasks, ignoring the geometric relationships among their learned subspaces. When task updates
occupy misaligned subspaces, projecting them onto a shared low-rank basis discards
task-relevant directions, a failure mode named \emph{subspace interference}. Through this
study, we show that the per-task projection error of a merged model is governed by the
principal-angle alignment between each task and the shared subspace, so that grouping
geometrically compatible tasks provably reduces this error within each group. Motivated by
this observation, we propose \textbf{Group-wise Rank-Aware Model Merging (GRAM)}, which uses the
principal-angle geometry of task vectors to decide which tasks should be merged together
and then performs a standard SVD-based merge within each coalition; the pipeline requires
only the fine-tuned weights at merge time. With only one additional model ($\nK{=}2$),
GRAM significantly mitigates subspace interference and consistently outperforms prior
merging methods across language understanding (GLUE with Flan-T5), 8-task CLIP vision, and
medical instruction tuning (Qwen2-7B-Instruct with four English/Chinese medical LoRAs).

\end{abstract}

\section{Introduction}
\label{sec:intro}

\begin{figure}[t]
\centering
\includegraphics[width=0.9\linewidth]{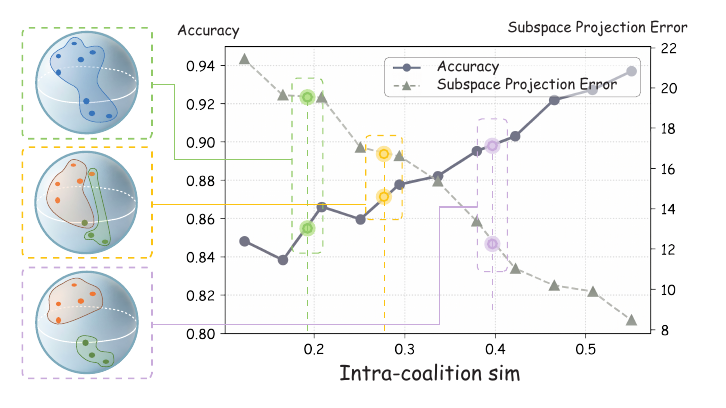}
\caption{%
  How tasks are grouped on the Grassmannian governs merging quality.
  \emph{Insets:} three partitions of the per-task low-rank subspaces visualised on the Grassmannian---all tasks merged into a single global coalition (top, low intra-coalition similarity), three groups (middle), and the two-group GRAM partition (bottom, high intra-coalition similarity).
  \emph{Curve:} as intra-coalition similarity increases, merged-model accuracy rises (left axis, solid) and the within-group subspace projection error $\sum_{\mathrm{pairs}\in G}(1{-}\cos^2\theta)$ falls (right axis, dashed).
  Highlighted markers tie each partition to its location on the trade-off curve: GRAM places geometrically compatible tasks in the same coalition, reducing the alignment penalty that global merging incurs.
}
\label{fig:overview}
\end{figure}

Model merging \citep{wortsman2022model,ilharco2022editing,tang2023concrete} aggregates multiple fine-tuned, task-specific models into a single deployable model without the need for additional training. This paradigm is highly attractive for multi-task deployment, as it circumvents the substantial costs of storing and serving independent models. The structured nature of weight updates—particularly in parameter-efficient fine-tuning (e.g., LoRA \citep{hu2022lora})—makes model merging highly effective. This approach is especially critical in privacy-constrained domains such as healthcare, where institutions independently fine-tune a shared backbone on local clinical data. Because data sharing across sites is strictly restricted, any unified multi-institution model must be constructed solely from the fine-tuned weights, precluding the use of the original training data.

Prior work on model merging spans parameter-space arithmetic with conflict resolution \citep{ilharco2022editing,yadav2023ties}, auxiliary signal coefficient learning \citep{yang2023adamerging,tang2023concrete}, and structured approximations exploiting low-rank or layer-wise properties \citep{gargiulo2025tsv,stoica2024knots,marczak2025iso,wei2025doge,sun2025lot, chen2025coefficients}. A pervasive limitation of these methods is their reliance on a uniform global merge: all tasks are compressed into a single shared space. Consequently, tasks whose learned parameter directions are least compatible with the majority suffer from disproportionate performance degradation—a phenomenon we term subspace interference. Furthermore, many existing techniques require access to task-specific calibration data or activations, rendering them inapplicable in strict privacy-preserving deployments where data cannot leave its source institution.

A critical signal overlooked by current methods is the inter-task geometry of the learned parameter subspaces. Existing techniques process updates individually or globally, systematically neglecting the pairwise geometric relationships between tasks. As illustrated in \Cref{fig:overview}, subspace similarity is highly predictive of merging compatibility: grouping tasks with high intra-coalition subspace similarity simultaneously increases merged-model accuracy and minimizes projection error. We observe distinct geometric signatures—such as cluster-outlier structures in GLUE tasks and block-diagonal affinities in vision tasks (\cref{app:exp_details}). Crucially, the outlier tasks are precisely those that disproportionately degrade standard global SVD merging, and thus benefit the most from group-wise isolation (\cref{sec:experiments}). Subspace geometry therefore carries predictive information that current methods fail to exploit.

Motivated by this geometric insight, we propose Group-wise Rank-Aware Model Merging (GRAM). GRAM targets the task-aware deployment setting, where task identity is known at inference. We reframe the merging objective from compressing $T$ models into a single model, to compressing $T$ models into $K \ll T$ models, trading a marginal increase in storage for a targeted, significant reduction in subspace interference. Crucially, GRAM is entirely data-free: the pipeline operates exclusively on the fine-tuned weights, requiring no source data, activations, or calibration sets at merge time. Our primary contributions are as follows:

\begin{enumerate}[leftmargin=*,itemsep=2pt,topsep=4pt]\item \textbf{A Geometry-Aware Merging Algorithm:} We introduce GRAM, which partitions tasks using Grassmannian $k$-means++ on chordal distances. Within each coalition, it dynamically merges models via a greedy adaptive rank allocator that pools singular values across members, achieving provable optimality under a global rank budget.\item \textbf{Theoretical Guarantees:} We rigorously bound the projection error of our clustering objective. We prove that the resulting partition yields a merge error within $O(\log K)$ of the optimal partition, and its expected error is strictly bounded below that of a standard global merge.\item \textbf{State-of-the-Art Data-Free Performance:} Extensive evaluations across language understanding, vision, and medical instruction-tuning demonstrate that GRAM consistently outperforms existing data-free baselines. Notably, on the medical benchmark, GRAM surpasses even data-dependent methods (e.g., RegMean, Fisher-Merge) that rely on per-task calibration data.\end{enumerate}
\section{Related Work}
\label{sec:related_work}

\paragraph{Model Merging.}
Model merging combines independently fine-tuned parameters into a single deployable model
without retraining. Parameter-level methods average fine-tuned weights or arithmetically
combine task vectors, sometimes with sparsification or sign/magnitude-based conflict resolution
\citep{wortsman2022model,ilharco2022editing,yadav2023ties,yu2024dare}. Coefficient-learning
methods learn per-task or per-layer mixing weights from auxiliary signals or coordinate-wise
masks \citep{yang2023adamerging,tang2023concrete}. Structured-update methods exploit the
low-rank or layered structure of fine-tuning updates, operating in the singular-value or
per-layer domain to reduce inter-task interference
\citep{gargiulo2025tsv,stoica2024knots,marczak2025iso,chen2025learning,wei2025doge,wang2024lines,sun2025lot,fu2025himerging}. Multi-model
approaches retain some per-task structure rather than collapsing to a single shared model
\citep{huang2024emr}. A common thread across these single-model approaches is that all tasks
participate in the same global merge, implicitly assuming sufficient subspace alignment---an
assumption that fails when tasks span diverse domains. We depart from it by measuring
inter-task subspace geometry directly and merging within geometrically compatible groups.

\paragraph{Task Grouping in Multi-Task Learning.}
Deciding which tasks should be trained together has a long history in multi-task learning.
\citet{zamir2018taskonomy} computationally map transfer relationships across 26 visual tasks, and
\citet{standley2020tasks} formalize the grouping problem and search over partitions to trade
accuracy against compute. Task2Vec~\citep{achille2019task2vec} represents each task by a
Fisher-information vector, while branched networks~\citep{vandenhende2019branched} and
mixture-of-experts models~\citep{chen2023modsquad} realize grouping through the architecture
itself. Affinity-based methods avoid combinatorial search:
TAG~\citep{fifty2021efficiently} estimates inter-task affinity from gradient signals in a single
training run, MTG-Net~\citep{song2022efficient} meta-learns a predictor of per-task gains across
task combinations, Grad-TAG~\citep{li2024scalable} scales affinity estimation via gradient
linearization, and~\citet{wang2024principled} cast grouping as a mathematical program with
theoretical guarantees. Complementary work on dynamic loss weighting~\citep{liu2022autolambda}
modulates task competition during training rather than partitioning tasks explicitly. Dynamic client cooperation~\citep{yang2026decentralized}
mitigates interference through adaptive coalition formation rather than explicitly partitioning tasks.
All of these approaches require training-time data and gradients, making them inapplicable to
the post-hoc merging setting where only frozen weights are available. Our grouping criterion is
data-free: it uses only the left singular subspaces of the task-specific weight updates.

\section{Preliminaries}
\label{sec:preliminaries}

\paragraph{Task vectors.}
We consider a pre-trained model with $\nL$ layers and weights
$W_0 = \{W_0^{(\ell)}\}_{\ell=1}^{\nL}$. After adapting the model to task
$i \in \{1,\dots,\nT\}$, we obtain fine-tuned weights $W_i$. Following the
task-vector view of model merging~\citep{ilharco2022editing,yadav2023ties,yu2024dare},
the adaptation is summarized by the per-layer \emph{task vector}
\begin{equation}
  \label{eq:task_vector}
  \dW_i^{(\ell)}
  \;=\; W_i^{(\ell)} - W_0^{(\ell)}
  \;\in\; \R^{\dout \times \din},
\end{equation}
i.e., the difference between the task-adapted and pre-trained weights at layer
$\ell$. The matrix formulation in \eqref{eq:task_vector} applies to \emph{linear}
layers --- MLP blocks and attention projection matrices --- which covers all merged
parameters in our experiments: the LoRA settings adapt only linear projections by
construction, and in the fully fine-tuned vision setting SVD-based merging is
applied to all 2-D weight matrices while the remaining non-matrix parameters
(LayerNorm gains, biases) are merged by simple averaging (\cref{app:impl_details}).
For models tuned with Low-Rank Adaptation (LoRA)~\citep{hu2022lora}, the
task vector inherits a rank-$r$ factorization
\begin{equation}
  \label{eq:lora}
  \tau_i^{(\ell)}
  \;=\; B_i^{(\ell)} A_i^{(\ell)},
  \qquad
  B_i^{(\ell)} \in \R^{\dout \times r},\;
  A_i^{(\ell)} \in \R^{r \times \din},\;
  r \ll \min(\dout,\din).
\end{equation}
\paragraph{SVD of task updates.}
The compact SVD of each update is
\begin{equation}
  \label{eq:svd}
  \dW_i^{(\ell)}
  = U_i^{(\ell)}\, \Sigma_i^{(\ell)}\, V_i^{(\ell)\top},
\end{equation}
where $U_i^{(\ell)} \in \R^{\dout \times r}$ contains the left singular vectors (output subspace),
$\Sigma_i^{(\ell)} \in \R^{r \times r}$ is a diagonal matrix of singular values
$\sigma_1 \geq \cdots \geq \sigma_r > 0$,
and $V_i^{(\ell)} \in \R^{\din \times r}$ contains the right singular vectors (input subspace).
We drop the layer index $(\ell)$ when the context is unambiguous.

\paragraph{Merging objective.}
Given $\nT$ task vectors, each of rank $r$ across $\nL$ layers,
the goal is to produce $\nK < \nT$ merged task vectors that minimize the total
reconstruction error across all tasks subject to a per-merge rank budget:
\begin{equation}
  \label{eq:merge_obj}
  \min_{\text{partition},\;\{\dWt_k\}_{k=1}^{\nK}}
  \;\sum_{i=1}^{\nT}
  \fnorm{\dW_i - \dWt_{\pi(i)}}^2
  \quad\text{s.t.}\quad
  \rank(\dWt_k) \le r \;\;\text{for all}\;k\in\{1,\dots,\nK\},
\end{equation}
where $\pi(i) \in \{1,\dots,\nK\}$ assigns task $i$ to its group $G_{\pi(i)}$ and
$\dWt_k \in \R^{\dout \times \din}$ is the merged update for group $G_k$. The model
deployed for group $G_k$ is, per layer,
\begin{equation}
  \label{eq:deployed_model}
  W_{G_k}^{(\ell)}
  \;=\; W_0^{(\ell)} + \dWt_k^{(\ell)},
\end{equation}
i.e., each of the $\nK$ deployed models is the pre-trained backbone plus its
group's merged task vector, and task $i$ is served by the model of group
$G_{\pi(i)}$.

\paragraph{Subspace interference.}
Let $B \in \R^{\dout \times r}$ be a shared output basis with orthonormal columns ($B^\top B = I_r$). Substituting the compact SVD $\dW_i = U_i\Sigma_i V_i^\top$, the projection error of task $i$ onto $\mathrm{col}(B)$ is
\begin{equation}
  \label{eq:proj_error}
  \mathcal{E}_i
  \;=\; \fnorm{(I - BB^\top)\,\dW_i}^2
  \;=\; \fnorm{(I - BB^\top)\,U_i\Sigma_i}^2,
\end{equation}
 We adopt $\mathcal{E}_i$ as a principled measure of inter-task interference as it directly measures the task-relevant information lost when updates are forced into a shared subspace. Crucially, $\mathcal{E}_i$ depends only on $U_i$ and $\Sigma_i$ --- not on $V_i$ --- so partitioning by output-subspace alignment captures the interference signal a shared basis can control.

To isolate the geometric contribution of subspace misalignment from the magnitude of
the task update, we drop the singular-value weighting and analyze
$\fnorm{(I - BB^\top)\,U_i}^2$. Every column of $U_i$ admits a unique orthogonal
decomposition into a component in $\text{col}(B)$ and a component in $\text{col}(B)^\perp$.
Applied column-wise,
\[
  U_i \;=\; \underbrace{BB^\top U_i}_{\in\,\text{col}(B)}
     \;+\; \underbrace{(I - BB^\top)\,U_i}_{\in\,\text{col}(B)^\perp}.
\]
The two summands live in orthogonal subspaces, so by the Pythagorean theorem (in the
Frobenius inner product),
$\fnorm{U_i}^2 = \fnorm{BB^\top U_i}^2 + \fnorm{(I - BB^\top)\,U_i}^2$. Since $B$ has
orthonormal columns, $\fnorm{BB^\top U_i}^2 = \fnorm{B^\top U_i}^2$, and
$\fnorm{U_i}^2 = r$. Rearranging,
\begin{equation}
  \label{eq:unweighted_geometry}
  \fnorm{(I - BB^\top)\,U_i}^2
  \;=\; r - \fnorm{B^\top U_i}^2.
\end{equation}
Let $\theta_1, \dots, \theta_r$ denote the principal angles between $\text{col}(U_i)$
and $\text{col}(B)$; throughout the paper we reserve $\theta$ for principal angles
and $k$ for merge groups (model weights are denoted $W$), and index principal
angles by $j$. A standard
result~\citep{golub1996matrix} states that the
singular values of $B^\top U_i$ are exactly $\{\cos\theta_j\}_{j=1}^{r}$, so
$\fnorm{B^\top U_i}^2 = \sum_{j=1}^{r}\cos^2\theta_j$. Substituting into
\eqref{eq:unweighted_geometry} and using $\sin^2\theta_j + \cos^2\theta_j = 1$,
\begin{equation}
  \label{eq:principal_angle}
  \fnorm{(I - BB^\top)\,U_i}^2
  \;=\; r - \sum_{j=1}^{r} \cos^2\theta_j
  \;=\; \sum_{j=1}^{r} \sin^2\theta_j,
\end{equation}
which is the squared chordal (projection F-) distance between the two subspaces. The
expression depends only on the relative geometry of $\text{col}(U_i)$ and
$\text{col}(B)$ and is invariant to the choice of basis. When the subspaces are
well-aligned ($\theta_j \approx 0$ for all $j$), the projection error vanishes. When
they are misaligned, important task directions are discarded by the merge. This motivates grouping tasks by
(singular-value-weighted) subspace compatibility before merging.

\section{Method: Group-wise Rank-Aware Model Merging}
\label{sec:method}

\begin{figure}[t]
\centering
\includegraphics[width=\linewidth]{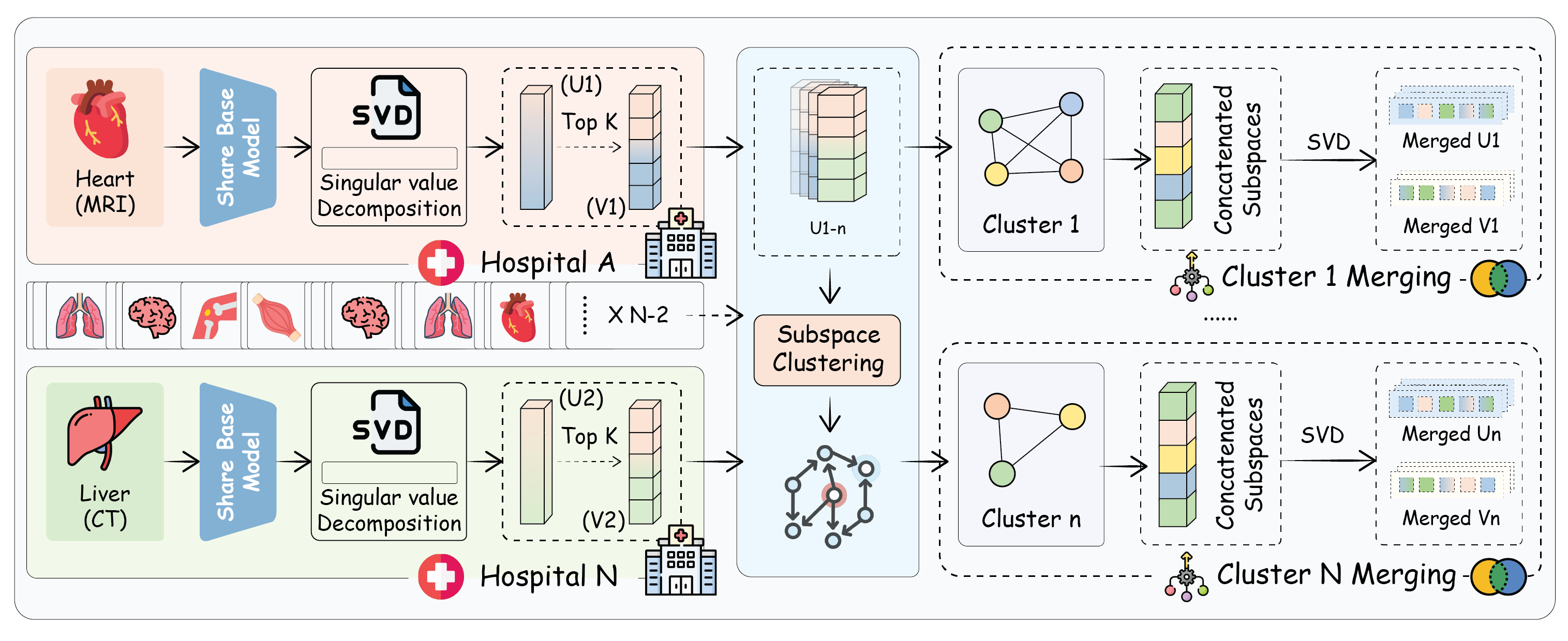}
\caption{The GRAM pipeline. Each task vector $\dW_i$ is decomposed via SVD to extract its top-$r$ left- and right-singular subspaces $(U_i, V_i)$ (left). The collection of $U_i$'s is fed into a subspace clustering step that partitions the $\nT$ tasks into $\nK$ coalitions by chordal distance (middle). Within each coalition, the selected top-$k_i$ singular components from member tasks are concatenated and re-orthonormalized via thin SVD to yield a single rank-$r$ merged update (right).}
\label{fig:gram_pipeline}
\end{figure}

Our goal is to merge $\nT$ task vectors into $\nK \ll \nT$ per-group updates in a way that protects geometrically incompatible tasks from mutual interference. 
To accomplish this, we decompose the pipeline (\cref{fig:gram_pipeline}) into three stages.
\textbf{(1) Quantifying pairwise subspace geometry.}
We compute the pairwise chordal distance between task output subspaces.
This quantity directly governs the projection error a shared basis would incur (\cref{lem:alignment_penalty}), 
so it is the natural input for the clustering step rather than coordinate-wise distances on flattened parameters.
\textbf{(2) Partitioning tasks into geometrically coherent coalitions.}
Treating each task's projection matrix $P_i = U_iU_i^\top$ as a point in Euclidean space, we run Grassmannian $k$-means++ on the chordal-distance matrix to obtain $\nK$ coalitions with low within-group and high between-group subspace distance, importing the $O(\log\nK)$ approximation guarantee of $k$-means++ into the Grassmannian setting.
\textbf{(3) Per-coalition low-rank merge.}
Within each coalition we concatenate the per-task top-$k_i$ singular triples under an adaptive rank allocator that distributes the rank budget to minimize the reconstruction error, 
then orthonormalize via a thin SVD to produce a single rank-$r$ merged update per coalition. 

\subsection{Chordal Subspace Distance}
\label{sec:similarity}

We measure the geometric compatibility between two task subspaces via the
\emph{squared chordal distance}.  At layer $\ell$, the chordal distance
between the column spaces of $U_i^{(\ell)}$ and $U_j^{(\ell)}$ is
\begin{equation}
  \label{eq:chordal_dist}
  d_c^2\!\bigl(U_i^{(\ell)},\, U_j^{(\ell)}\bigr)
  \;=\; r \,-\, \fnorm{U_i^{(\ell)\top} U_j^{(\ell)}}^2
  \;=\; \sum_{s=1}^{r} \sin^2\theta_s,
\end{equation}
where $\theta_1,\dots,\theta_r$ are the principal angles between
$\mathrm{col}(U_i^{(\ell)})$ and $\mathrm{col}(U_j^{(\ell)})$ (indexed by $s$ here
since $j$ denotes a task). By \cref{eq:principal_angle}, this quantity directly governs the projection error
incurred when two tasks share a merged basis, making it the natural metric for
deciding which tasks to group. Furthermore, we show that the chordal distance on the Grassmannian is equivalent to the squared Frobenius distance between the orthogonal projectors $P_i^{(\ell)} := U_i^{(\ell)}U_i^{(\ell)\top}$ up to a factor of $2$ (\cref{app:chordal_eucl_proof}), so off-the-shelf $k$-means++ on $\{P_i^{(\ell)}\}$ inherits its $O(\log\nK)$ approximation guarantee in the chordal geometry (\cref{app:theory}).

\subsection{Grassmannian $k$-Means++ Clustering}
\label{sec:clustering}

To obtain a single distance per task pair, we average the per-layer chordal
distance across layers:
\begin{equation}
  \label{eq:model_dist}
  \bar{d}_c^2(i,j)
  \;=\; \frac{1}{\nL}\sum_{\ell=1}^{\nL}
        d_c^2\!\bigl(U_i^{(\ell)},\, U_j^{(\ell)}\bigr).
\end{equation}
We then run $k$-means++~\citep{arthur2007kmeans} on $\bar{d}_c^2$ to partition
the $\nT$ tasks into $\nK$ groups $\{G_1,\dots,G_\nK\}$.
The number of groups $\nK$ is a user-specified \emph{compression budget}
that trades deployment cost ($\nK$ deployed models) against merging fidelity.
In our experiments (\cref{sec:experiments}), $\nK{=}2$ provides the best
cost--accuracy trade-off on all benchmarks.

\subsection{Per-Group Low-Rank Merge}
\label{sec:group_merge}

For each group $G_k$ and layer $\ell$, we merge the constituent task
vectors into a single rank-$r$ update. The procedure operates on all three
SVD components---left singular vectors $U_i$, singular values $\sigma_i$,
and right singular vectors $V_i$---and allocates rank budget adaptively
rather than uniformly across tasks.

\paragraph{Adaptive rank allocation.}
A naïve approach gives each of the $|G_k|$ tasks an equal share
$r / |G_k|$ of the rank budget. Since intrinsic rank varies sharply across
tasks and layers in practice (\cref{fig:singular_value_count}, \cref{app:rank_heterogeneity}),
this uniform allocation wastes capacity on compact tasks and starves complex ones.

We instead formulate rank allocation as a discrete optimization problem.
Let $\Delta W_i = U_i \Sigma_i V_i^\top$ be the SVD of task $i$'s
update, and let $r_i$ denote the number of singular directions allocated
to task $i$. The per-task reconstruction error after keeping the top
$r_i$ components is
\begin{equation}
  \label{eq:task_recon_error}
  \mathcal{L}_i(r_i)
  = \bigl\lVert (I - B_i^{(r_i)} {B_i^{(r_i)}}^\top)\,\Delta W_i \bigr\rVert_{\mathrm{F}}^2
  = \sum_{j > r_i} \sigma_{i,j}^2,
\end{equation}
where $B_i^{(r_i)} \in \R^{\dout \times r_i}$ contains the top-$r_i$
left singular vectors of $\Delta W_i$, and the second equality follows
from the Eckart--Young theorem. We seek the allocation that minimizes
total reconstruction error under a global budget:
\begin{equation}
  \label{eq:rank_alloc}
  \min_{\{r_i\}}
  \;\sum_{i \in G_k} \sum_{j > r_i} \sigma_{i,j}^2
  \qquad \text{s.t.}\quad \sum_{i \in G_k} r_i \le r.
\end{equation}
The marginal gain of assigning one additional rank slot to task $i$ is
$\mathcal{L}_i(r_i) - \mathcal{L}_i(r_i + 1) = \sigma_{i,\,r_i+1}^2$.
Since singular values are sorted in descending order within each task,
the optimal strategy is \emph{greedy}: pool all candidate singular
values $\{\sigma_{i,j}^2\}$ across tasks in the group, select the
top-$r$ entries, and count how many are drawn from each task. This
yields optimal per-task allocations $r_1, \dots, r_{|G_k|}$ with
$\sum_i r_i = r$, ensuring that rank budget flows to the directions
that most reduce reconstruction error.

\paragraph{Coalition merge.}
Given the per-task allocations $\{r_i\}$, we follow TSV-Merge~\citep{gargiulo2025tsv}:
concatenate the selected top-$r_i$ SVD components across tasks into
$\bar{U} \in \R^{\dout \times r}$, $\bar{s} \in \R^{r}$, and
$\bar{V} \in \R^{r \times \din}$, sort globally by singular value, replace
$\bar{U}$ and $\bar{V}^\top$ with their orthogonal polar factors $\widehat U,\widehat V$ (rather than computing an SVD of the joint operator), and reconstruct
$\dWt_k^{(\ell)} = \widehat U\,\diag(\bar{s})\,\widehat V \in \R^{\dout \times \din}$
as the task vector for group $G_k$ at layer $\ell$.

\paragraph{Theoretical guarantee for grouping.}
The grouping step is backed by a deterministic projection-error guarantee: with optimal per-group bases, any partition into $\nK$ groups achieves total reconstruction error no larger than a single global rank-$r$ basis fit to all tasks (\cref{thm:grouping} in \cref{app:theory}; the proof is a direct optimality argument independent of the clustering rule). The guarantee explains \emph{why} the empirical gains in \cref{sec:experiments} concentrate on tasks geometrically furthest from the bulk---exactly the tasks for which a partition-aware basis improves most over the global optimum, making the inequality strict and large in magnitude---while the alignment-penalty machinery (\cref{lem:alignment_penalty}) and the $k$-means++ partition error chain in \cref{app:theory} quantify \emph{how much} grouping helps.

\section{Experiments}
\label{sec:experiments}

\subsection{Setup}

We evaluate GRAM on three benchmarks: GLUE \citep{wang2018glue} with Flan-T5-Base/Large \citep{chung2024scaling}, the 8-task vision benchmark with fully fine-tuned CLIP ViT-B/32 and ViT-L/14, and medical instruction tuning on Qwen2-7B-Instruct with four LoRA experts (\cref{sec:qwen_medical}). All methods operate on the same pre-existing fine-tuned models---from FusionBench \citep{tang2025fusionbench} for GLUE and from the TSV-Merge release \citep{gargiulo2025tsv} for vision---and are re-run under a unified protocol (\cref{app:exp_details}), so cross-method differences come only from merging. We report validation accuracy on GLUE (Spearman $\rho{\times}100$ for STS-B), test accuracy on vision, and the metrics described in \cref{sec:qwen_medical} for the medical benchmark. All experiments use task-aware inference: task identity is known at test time and selects the group-specific adapter. This is a preference rather than a requirement (\cref{sec:ablation}).

For GLUE and vision, we compare against seven baselines: Weight Averaging \citep{wortsman2022model}, Task Arithmetic \citep{ilharco2022editing}, TIES-Merging \citep{yadav2023ties}, AdaMerging++ \citep{yang2023adamerging}, Iso-C and Iso-CTS \citep{marczak2025iso}, and TSV-Merge \citep{gargiulo2025tsv}. The GLUE results in \cref{tab:glue_compact} additionally include Concrete TA, Concrete AM \citep{tang2023concrete}, and DOGE TA \citep{wei2025doge}. The CLIP ViT-B/16 results in \cref{app:vit_b16_ft} additionally include LiNeS-Merge \citep{wang2024lines} and LOT-Merging \citep{sun2025lot} for completeness across recent structured-update baselines. The medical benchmark (\cref{sec:qwen_medical}) additionally compares against Fisher-Merge \citep{matena2022merging} and RegMean \citep{jin2022dataless} as data-using references. Ablation includes random-equipartition controls under the same chordal metric.

\subsection{Language Understanding}

\begin{table}[t]
  \centering
\caption{%
  GLUE results on Flan-T5-Base and Flan-T5-Large (LoRA $r{=}16$):
  validation accuracy, Spearman $\rho{\times}100$ for STS-B.
  Each cell reports Base\,/\,Large in that order.
  \textbf{Bold}: best merged result per column for each variant.
  Individual: per-task fine-tuned model (oracle upper bound).
}
\label{tab:glue_compact}
\scriptsize
\setlength{\tabcolsep}{3pt}
\renewcommand{\arraystretch}{1.05}
\resizebox{\textwidth}{!}{
\begin{tabular}{l c c c c c c c c c}
\toprule
Method & CoLA & MNLI & MRPC & QNLI & QQP & RTE & SST-2 & STS-B & Avg. \\
\midrule
Individual (oracle) & 69.1 / 80.1 & 82.7 / 88.5 & 85.5 / 89.2 & 90.9 / 94.4 & 84.0 / 87.1 & 84.4 / 91.3 & 92.9 / 95.2 & 87.4 / 90.9 & 84.6 / 89.6 \\
\midrule
Weight Averaging     & \textbf{69.7} / 74.6 & 59.7 / 84.3 & 78.9 / 84.1 & 90.1 / 92.8 & 83.8 / 86.2 & 80.5 / 85.9 & 91.2 / 94.8 & 72.0 / 88.0 & 78.2 / 86.3 \\
Task Arithmetic      & 69.5 / 76.9 & 59.0 / 85.4 & 78.7 / 85.3 & 90.1 / \textbf{93.9} & 83.9 / 85.8 & 79.4 / 88.1 & 91.5 / 95.2 & 72.9 / 87.8 & 78.1 / 87.3 \\
Concrete TA          & 69.1 / 76.6 & 58.1 / 86.3 & 78.4 / 86.3 & 89.9 / \textbf{93.9} & 83.5 / 85.9 & 79.4 / \textbf{88.4} & 91.6 / 95.2 & 73.4 / 87.9 & 78.0 / 87.5 \\
DOGE TA              & 69.1 / 78.4 & 71.9 / 88.1 & 80.9 / 86.5 & \textbf{90.3} / 93.8 & 83.5 / \textbf{86.3} & 79.8 / 87.7 & 92.5 / 95.1 & 71.1 / 87.7 & 79.9 / 88.0 \\
TIES-Merging         & 68.3 / 77.1 & 56.3 / 85.1 & 79.4 / 86.3 & 89.8 / \textbf{93.9} & 83.7 / 86.0 & 79.4 / 87.7 & 91.6 / 95.1 & 71.2 / 88.0 & 77.5 / 87.4 \\
AdaMerging++         & 69.1 / 76.7 & 60.3 / 87.6 & 78.4 / 84.8 & 90.0 / 93.8 & 83.6 / 85.9 & 79.1 / 88.1 & 91.6 / 95.2 & 74.1 / 88.6 & 78.3 / 87.6 \\
Concrete AM          & 69.0 / 76.1 & 59.4 / 87.7 & 80.1 / 85.5 & 89.9 / 93.8 & 82.9 / 85.9 & 79.1 / 88.1 & 91.7 / 95.4 & 75.4 / 87.1 & 78.5 / 87.8 \\
TSV-Merge            & 69.2 / 79.2 & 74.7 / 86.0 & 80.9 / 87.3 & 90.2 / 93.6 & 83.6 / 85.6 & 78.7 / 87.0 & 92.7 / 95.3 & 83.2 / 88.7 & 81.7 / 87.8 \\
Iso-C                & 69.1 / 73.8 & 59.9 / 64.2 & 77.9 / 82.4 & 88.9 / 92.2 & 82.9 / 86.0 & \textbf{80.9} / 82.7 & 91.5 / 92.8 & 66.1 / 86.4 & 77.2 / 82.5 \\
Iso-CTS              & 69.1 / 73.8 & 58.3 / 59.1 & 77.0 / 82.1 & 88.5 / 91.8 & 82.2 / 85.8 & 80.5 / 82.3 & 91.4 / 92.4 & 64.0 / 86.0 & 76.4 / 81.7 \\
\midrule
\textbf{GRAM} (ours) & 69.5 / \textbf{79.3} & \textbf{82.8} / \textbf{89.0} & \textbf{84.3} / \textbf{88.7} & 90.2 / 93.8 & \textbf{84.0} / 85.7 & 79.8 / 88.1 & \textbf{93.2} / \textbf{95.4} & \textbf{86.7} / \textbf{89.9} & \textbf{83.8} / \textbf{88.7} \\
\bottomrule
\end{tabular}
}
\end{table}

\Cref{tab:glue_compact} presents GLUE results on Flan-T5-Base and Flan-T5-Large (LoRA $r{=}16$). We fix $\nK{=}2$ throughout to keep the budget overhead minimal---a single extra deployed model relative to the $\nK{=}1$ baselines---and ask whether this small extra resource yields meaningful gains. On Flan-T5-Base, GRAM reaches 83.8\% average, $+2.1$ over the strongest baseline TSV-Merge (81.7\%). On Flan-T5-Large, the gap between methods narrows in the higher-capacity regime, but GRAM still leads the field at 88.7\%, $+0.9$ over TSV-Merge (87.8\%)---all with only $2$ deployed models and no per-task calibration data.

The per-task profile is consistent with the alignment-penalty mechanism of \cref{lem:alignment_penalty}: tasks at large principal angles from the bulk cluster bear the largest projection error under a global merge and stand to recover the most under per-coalition merging. MNLI and STS-B sit furthest from the \{MRPC, QNLI, QQP, RTE\} cluster in subspace geometry (\cref{app:exp_details}), and they are the two tasks where GRAM recovers the most relative to TSV-Merge at both scales (MNLI $74.7{\to}82.8$ base, $86.0{\to}89.0$ large; STS-B $83.2{\to}86.7$ base, $88.7{\to}89.9$ large). On QQP, GRAM attains the highest base accuracy ($84.0$); on SST-2 it is within $0.1$ of the strongest baseline.

\subsection{Vision}

\begin{table}[t]
  \centering
\caption{%
  Eight-task vision results on fully fine-tuned CLIP ViT-B/32 and ViT-L/14: test accuracy (\%).
  Each cell reports B/32\,/\,L/14 in that order.
  \textbf{Bold}: best merged result per column for each variant.
  Individual: per-task fine-tuned model (oracle upper bound).
}
\label{tab:clip_vit_ft_compact}
\scriptsize
\setlength{\tabcolsep}{3pt}
\renewcommand{\arraystretch}{1.05}
\resizebox{\textwidth}{!}{
\begin{tabular}{l c c c c c c c c c}
\toprule
Method & MNIST & SVHN & GTSRB & DTD & Cars & RESISC45 & EuroSAT & SUN397 & Avg. \\
\midrule
Individual (oracle) & 99.7 / 99.8 & 97.6 / 98.1 & 99.1 / 99.2 & 97.7 / 98.5 & 78.2 / 92.2 & 95.9 / 97.3 & 99.8 / 99.8 & 75.3 / 82.0 & 92.9 / 95.9 \\
\midrule
Weight Averaging     & 86.5 / 96.8 & 62.7 / 77.7 & 53.3 / 68.9 & 54.7 / 66.8 & 62.6 / 81.7 & 71.9 / 82.3 & 71.0 / 88.7 & 65.1 / 71.5 & 66.0 / 79.3 \\
Task Arithmetic      & 96.9 / 99.0 & 78.1 / 86.0 & 68.2 / 84.9 & 59.6 / 72.2 & 56.2 / 81.9 & 68.1 / 86.3 & 76.0 / 91.9 & 57.3 / 72.9 & 70.1 / 84.4 \\
TIES-Merging         & 97.0 / 99.0 & 80.5 / 87.6 & 69.3 / 83.7 & 66.2 / 77.5 & 61.1 / 83.3 & 73.2 / 87.7 & 79.4 / 93.8 & 62.7 / 74.0 & 73.7 / 85.8 \\
AdaMerging++         & 98.2 / 99.0 & 90.1 / 90.2 & 73.2 / 95.7 & 81.3 / 93.7 & 69.2 / 88.6 & 78.2 / 91.5 & 93.9 / 97.2 & 65.7 / 74.7 & 81.2 / 91.3 \\
TSV-Merge            & 99.3 / \textbf{99.6} & 91.5 / 95.1 & 91.8 / 96.3 & 83.8 / 93.4 & 69.5 / 89.3 & 84.8 / 93.2 & 94.1 / 98.5 & 65.6 / 76.8 & 85.0 / 92.8 \\
Iso-C                & 98.9 / 99.3 & 87.0 / 90.5 & 91.0 / 95.8 & 88.1 / 96.1 & 71.0 / 91.4 & 87.0 / 94.7 & 92.4 / 98.8 & 69.1 / 79.7 & 85.6 / 93.3 \\
Iso-CTS              & 98.9 / 99.4 & 84.9 / 91.2 & 91.6 / 96.6 & 89.0 / \textbf{97.0} & 71.4 / \textbf{91.9} & 88.4 / \textbf{95.4} & 92.5 / 99.1 & \textbf{69.4} / \textbf{80.2} & 85.8 / 93.9 \\
\midrule
\textbf{GRAM} (ours) & \textbf{99.4} / \textbf{99.6} & \textbf{94.2} / \textbf{96.4} & \textbf{94.6} / \textbf{98.3} & \textbf{90.4} / 95.5 & \textbf{73.1} / 90.6 & \textbf{89.4} / 94.4 & \textbf{99.1} / \textbf{99.3} & 69.1 / 78.2 & \textbf{88.7} / \textbf{94.0} \\
\bottomrule
\end{tabular}
}
\end{table}

\Cref{tab:clip_vit_ft_compact} reports per-task accuracy on fully fine-tuned CLIP ViT-B/32 and
ViT-L/14. GRAM attains the highest average accuracy on both architectures (88.7\% B/32,
94.0\% L/14), surpassing Iso-CTS (85.8\% / 93.9\%) and TSV-Merge (85.0\% / 92.8\%). The pairwise projection-error matrix $\sum_s\sigma_{i,s}^2\sin^2\theta_s$ on the 8 vision tasks (App.\ \ref{app:vit_proj_error}) shows a clear block-diagonal structure under $\nK{=}2$, with high-chordal-distance pairs routed into different groups. The accuracy gains track this structure: GRAM is strongest on GTSRB (94.6 / 98.3) and EuroSAT (99.1 / 99.3), which sit furthest from the visual-recognition cluster (DTD/Cars/RESISC45/SUN397) and pay the largest off-diagonal penalty under $\nK{=}1$. The vision gap to the next-best baseline ($+2.9$ on B/32) exceeds the NLU gap ($+2.1$ on Flan-T5-Base), suggesting vision tasks present more heterogeneous subspaces where partitioning offers more benefit.

\subsection{Medical Instruction Tuning}
\label{sec:qwen_medical}

\begin{table}[t]

  \centering
\caption{%
  Medical merging on Qwen2-7B-Instruct with four LoRA experts.
  MedQA / CMExam: accuracy;
  HealthcareMagic / cMedQA2: BLEU-4 / ROUGE-1/2/L.
  \textit{Avg Impr.}: mean relative improvement over Pretrained across $10$ columns;
  \textit{Avg Rank.}: mean per-column rank ($1$ = best).
  \textbf{Bold} / \underline{underline}: best / second-best per column among merging methods.
  Individual: per-task fine-tuned oracle.
}
\label{tab:qwen_medical_main}
\renewcommand{\arraystretch}{1.25}
\resizebox{\textwidth}{!}{
\begin{tabular}{l c c cccc cccc cc}
\toprule
& MedQA & CMExam & \multicolumn{4}{c}{HealthcareMagic} & \multicolumn{4}{c}{cMedQA2} & & \\
\cmidrule(lr){4-7} \cmidrule(lr){8-11}
Method & Acc. & Acc. & BLEU-4 & ROUGE-1 & ROUGE-2 & ROUGE-L & BLEU-4 & ROUGE-1 & ROUGE-2 & ROUGE-L & Avg Impr. & Avg Rank. \\
\midrule
Pretrained & 51.92 & 74.77 & 1.05 & 25.90 & 3.27 & 12.65 & 9.77 & 40.30 & 19.05 & 20.94 & -- & -- \\
Individual (oracle) & 60.25 & 88.63 & 5.43 & 30.60 & 8.07 & 18.48 & 11.17 & 41.71 & 19.23 & 24.02 & +69.62\% & 4.55 \\
\midrule
Weight Averaging & 56.64 & 85.63 & 4.49 & 30.71 & 7.13 & 17.70 & 10.78 & 41.26 & 20.08 & 22.28 & +55.23\% & 7.65 \\
Task Arithmetic & 57.97 & \underline{87.69} & 5.00 & 30.33 & 7.72 & 18.11 & 5.55 & 32.70 & 12.02 & 19.15 & +49.40\% & 7.05 \\
TIES-Merging & 58.05 & 87.69 & 4.23 & \underline{30.72} & 7.03 & 17.58 & 11.30 & 42.06 & 20.60 & 22.85 & +54.18\% & 6.10 \\
Fisher-Merge & \underline{60.25} & 85.65 & 2.30 & 27.87 & 4.58 & 14.59 & 12.01 & \textbf{44.06} & 21.36 & 24.04 & +27.18\% & 6.25 \\
RegMean & \textbf{60.80} & \textbf{88.21} & \underline{5.51} & \textbf{31.19} & \underline{8.21} & \underline{18.58} & 11.39 & 41.75 & 20.25 & 23.05 & \underline{+71.48\%} & \underline{3.40} \\
Iso-C & 58.05 & 86.68 & 4.52 & 29.58 & 7.12 & 17.60 & 9.78 & 39.60 & 17.93 & 23.65 & +53.47\% & 7.85 \\
Iso-CTS & 57.82 & 86.74 & 4.54 & 29.73 & 7.20 & 17.77 & 11.97 & 42.84 & 20.28 & \textbf{25.26} & +59.11\% & 5.40 \\
TSV-Merge & 56.64 & 86.72 & 5.12 & 30.61 & 7.71 & 18.10 & \underline{12.55} & \underline{43.84} & \textbf{21.84} & 24.40 & +67.81\% & 4.35 \\
\midrule
\textbf{GRAM} (ours) & 58.52 & 87.60 & \textbf{5.62} & 30.70 & \textbf{8.30} & \textbf{18.68} & \textbf{12.67} & 43.68 & \underline{21.49} & \underline{25.01} & \textbf{+75.54\%} & \textbf{2.40} \\
\bottomrule
\end{tabular}
}
\end{table}

We set up a case study of the medical scenario from \cref{sec:intro}: four publicly available LoRA experts on Qwen2-7B-Instruct, fine-tuned for distinct medical sub-domains across two task families and two languages: MedQA (English MCQA), CMExam (Chinese MCQA), HealthcareMagic (English doctor-patient dialogue), and cMedQA2 (Chinese doctor-patient dialogue). MCQA tasks are scored by accuracy; open-ended tasks by BLEU-4 and ROUGE-1/2/L, yielding $10$ heterogeneous metric columns. Like the data-free baselines (Weight Averaging, Task Arithmetic, TIES-Merging, TSV-Merge, Iso-C, Iso-CTS), GRAM requires no calibration data. We additionally include Fisher-Merge \citep{matena2022merging} and RegMean \citep{jin2022dataless} as data-using references that consume a standard $64$-sample-per-task calibration set.

\Cref{tab:qwen_medical_main} reports per-metric results. With experts spanning two languages and two task families (MCQA / dialogue), plain averaging visibly interferes across the gap---Task Arithmetic regresses on cMedQA2 BLEU-4 ($5.55$ vs Pretrained $9.77$) and on its ROUGE columns. GRAM's $\nK{=}2$ coalitions route subspace-distant pairs into separate groups and attain the best Avg Impr ($+75.54\%$) and Avg Rank ($2.40$) of any merging method, ahead of even the data-using RegMean ($+71.48\%$, $3.40$) with its $64$-sample-per-task calibration set. Gains concentrate on the open-ended generation columns (both BLEU-4 fields; HealthcareMagic ROUGE-2 / ROUGE-L), where preserving task-specific singular directions outweighs calibrated reweighting; on the two MCQA accuracy columns---the regime where calibration-based methods are strongest---GRAM trails the leader by under one point. The summary-metric edge also extends past the per-task Individual oracle ($+69.62\%$, $4.55$), consistent with recent findings that LLM merging can match or exceed fine-tuned baselines on this benchmark family \citep{fu2025himerging}.

\subsection{Ablation Studies}
\label{sec:ablation}

\begin{wraptable}{r}{0.45\textwidth}
  \vspace{-\baselineskip}
  \centering
  \caption{%
    Effects of the proposed modules. Subscripts (\textcolor{red!70!black}{red}) report the change relative to full GRAM.
  }
  \label{tab:ablation}
  \small
  \setlength{\tabcolsep}{4pt}
  \begin{tabular}{lcc}
    \toprule
    Model & GLUE & ViT-B/32 \\
    \midrule
    \textbf{GRAM} & \textbf{83.8} & \textbf{88.8} \\
    \midrule
    $-$ PA-Sim     & $80.8_{\,\textcolor{red!70!black}{-3.0}}$ & $88.3_{\,\textcolor{red!70!black}{-0.5}}$ \\
    $-$ Coalition  & $81.7_{\,\textcolor{red!70!black}{-2.1}}$ & $85.4_{\,\textcolor{red!70!black}{-3.4}}$ \\
    $-$ Adapt-Rank & $82.2_{\,\textcolor{red!70!black}{-1.6}}$ & $88.7_{\,\textcolor{red!70!black}{-0.1}}$ \\
    \bottomrule
  \end{tabular}
  \vspace{-\baselineskip}
\end{wraptable}
\paragraph{Effects of each module.}
\Cref{tab:ablation} evaluates the contribution of each GRAM component to overall performance. Starting from full GRAM, we remove one component at a time and report the benchmark average for both PEFT merging (Flan-T5-Base on GLUE) and full-model merging (CLIP ViT-B/32 on the 8-task vision benchmark). Removing \textbf{PA-Sim} replaces the chordal subspace metric with an alternative that ignores cross-task geometry, breaking the alignment-penalty signal that drives the partition. Removing \textbf{Coalition} collapses GRAM to $\nK{=}1$ and recovers the global TSV-Merge baseline. Removing \textbf{Adapt-Rank} substitutes a uniform $r/|G_k|$ budget for the adaptive allocator, wasting capacity on sharp-spectrum tasks and starving long-tail ones. Each component is crucial, but their relative importance flips across modalities: on GLUE, PA-Sim drives the largest drop ($-3.0$), followed by Coalition ($-2.1$) and Adapt-Rank ($-1.6$); on ViT-B/32, Coalition dominates ($-3.4$) while PA-Sim ($-0.5$) and Adapt-Rank ($-0.1$) are small --- the heterogeneous LoRA spectra make adaptive allocation vital for PEFT but largely cosmetic for the more uniform full-fine-tuned vision deltas. With all components combined, GRAM boosts TSV-Merge by $+2.1$ on GLUE and $+3.8$ on ViT-B/32, demonstrating the complementarity of these components.
\paragraph{Sensitivity to the number of groups $\nK$.}
\Cref{tab:k_sweep} evaluates GRAM for every $\nK$ from $1$ to $\nT{=}8$ on Flan-T5-Base and Flan-T5-Large under one unified protocol (caption). At both scales the cost--accuracy frontier has a sharp knee at $\nK{=}2$: moving from Weight Averaging---the simplest single-model baseline---to GRAM at $\nK{=}2$ gains $+5.6$ (Base) and $+2.4$ (Large) points, capturing $87\%$ and $73\%$ of the respective gaps to the no-merging oracle, for a single extra rank-16 adapter ($0.7\%$ / $0.6\%$ of backbone parameters). Raising $\nK$ from $2$ all the way to $7$ adds at most $+0.1$ (Base) and $+0.6$ (Large) on top---on Base the average is not even monotone, dipping at $\nK{=}4$. Most of the achievable gain is therefore already available at small $\nK$: GRAM's geometry-aware clustering identifies compatible coalitions efficiently enough that a large $\nK$ is not needed to capture most of the gains in this setting. The practical rule is budget-first---deploy the smallest $\nK$ the serving budget allows; $\nK{=}2$, our default, captures the bulk of the gain at the tightest possible budget.

\begin{table}[t]
  \centering
\caption{%
  Clustering-algorithm ablation (Flan-T5-Base / GLUE): four ways of producing the
  $\nK{=}2$ partition from the \emph{same} subspace geometry, plus the random
  control, under identical merging and evaluation.
}
\label{tab:clustering_ablation}
\small
\setlength{\tabcolsep}{5pt}
\resizebox{\textwidth}{!}{
\begin{tabular}{l l c}
\toprule
Partitioning of the same subspace geometry & $\nK{=}2$ partition & Avg. \\
\midrule
GRAM & \{CoLA, MRPC, QNLI, QQP, RTE, SST-2\} $\mid$ \{MNLI, STS-B\} & \textbf{83.8} \\
Agglomerative (average linkage)          & identical partition & 83.8 \\
Max-modularity communities (exact, all 127 bipartitions) & \{CoLA, MNLI, SST-2, STS-B\} $\mid$ \{MRPC, QNLI, QQP, RTE\} & 82.9 \\
Random equipartition                     & random $4/4$ & 82.4 \\
No grouping ($\nK{=}1$, re-merged)       & all tasks & 81.9 \\
\bottomrule
\end{tabular}
}
\end{table}

\paragraph{Robustness of the partition.}
\Cref{tab:clustering_ablation} compares four ways of producing the $\nK{=}2$ partition from the \emph{same} subspace geometry. The partition is stable: it is identical across 10 initialization seeds at both $\nK{=}2$ and $\nK{=}3$ ($10/10$), and an independent hierarchy-based algorithm (agglomerative, average linkage) reproduces GRAM's partitions exactly at $\nK \in \{2,3,5,6,7\}$, with the exact max-modularity communities coinciding for every $\nK \ge 3$. The partition structure is thus a property of the geometry, not of a particular clustering algorithm---the gain comes from the subspace signal itself. The one disagreement is informative: at $\nK{=}2$ the degree-corrected modularity objective prefers a $4/4$ split that pulls MNLI into the low-similarity cluster and scores $0.9$ lower ($82.9$), whereas the solution used by GRAM respects the raw block structure of the similarity matrix and keeps the stronger $6/2$ partition.

\paragraph{Deployment without oracle task identity.}
Task-aware inference is a preference, not a requirement: a lightweight Mahalanobis task mapper \citep{lee2018unified}, as deployed in SKILL \citep{ge2023lightweight}---class-conditional Gaussians in the frozen base encoder's feature space, $5$ training samples per class, no gradient training, no test-time task labels---routes $100.0\%$ of validation examples to the correct group on Flan-T5-Base / GLUE at $\nK{=}2$, so task-agnostic GRAM \emph{equals} oracle GRAM ($83.8$ on every task); group-routing accuracy stays ${\ge}99.98\%$ for every $\nK$. Its only errors (QQP$\leftrightarrow$MRPC, $0.18\%$ of QQP) are absorbed by the coalition structure, since GRAM groups together precisely the similar tasks a router confuses (\cref{app:router}).

\begin{wrapfigure}{r}{0.48\textwidth}
  \vspace{-\baselineskip}
  \centering
  \includegraphics[width=\linewidth]{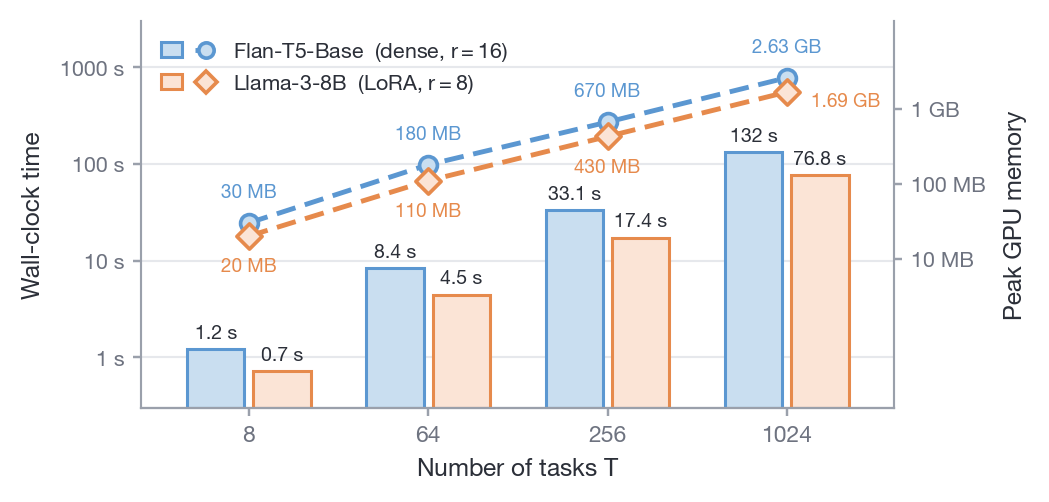}
  \caption{%
    Grouping cost vs.\ number of tasks $\nT$ on one RTX 3080 (bars: wall-clock
    time; lines: peak GPU memory). $\nT{=}8$ is the real task set; larger $\nT$
    uses synthetic tasks with the same layer shapes and rank.}
  \label{fig:scaling_T}
  \vspace{-\baselineskip}
\end{wrapfigure}

\paragraph{Scalability of the grouping step.}
A natural concern is whether grouping becomes a bottleneck when the number of tasks is
large. \Cref{fig:scaling_T} measures the end-to-end cost of grouping on a single RTX 3080
as $\nT$ grows from 8 to 1024. No benchmark offers thousands of adapters, so beyond our
real tasks we use perturbed copies of them; grouping cost depends only on adapter shapes,
not on their values, so these copies are a faithful proxy. Grouping is cheap: the real
8-task setting takes about one second, and 1024 tasks finish within a few minutes. Both
time and memory grow roughly linearly with $\nT$, because decomposing each adapter
dominates the cost and the pipeline processes one layer at a time. Llama-3-8B appears
 cheaper than Flan-T5 only because its LoRA factors are used directly rather than dense
 per-layer deltas.

\section{Conclusion}
\label{sec:conclusion}

GRAM partitions tasks by chordal subspace distance via Grassmannian $k$-means++ and merges within each group with a standard SVD-based procedure, avoiding the projection loss of global merging on misaligned subspaces. On GLUE with Flan-T5-Base/Large, GRAM ($\nK{=}2$) reaches 83.8\%/88.7\% average, matching the strongest baseline at the cost of one extra deployed model. On 8-task vision (CLIP ViT-B/32, ViT-L/14), GRAM reaches 88.7\%/94.0\%, the highest among compared methods. On the medical benchmark (Qwen2-7B-Instruct with four LoRA experts), GRAM attains the best Avg Rank without using any per-task calibration data. Gains concentrate on tasks whose subspaces are most dissimilar from the majority---precisely the tasks that the alignment penalty of \cref{lem:alignment_penalty} predicts will benefit most from grouping.

\paragraph{Limitations.}
Our results have several scope limits: (i)~single-seed reporting, so small Flan-T5-Large margins should be read with caution; (ii)~one model family per modality (Flan-T5 for NLU, CLIP for vision, Qwen2 for medical), and the medical evaluation is a case study on public benchmarks rather than a real federated experiment from multiple sites; (iii)~the theory analyzes idealized objects (optimal basis, Grassmannian $k$-means objective) rather than GRAM's exact column-stacked merge with adaptive allocation, so it should be read as geometric motivation rather than end-to-end optimality guarantee.

\paragraph{Future work.}
Scaling the task-agnostic router of \cref{app:router} beyond small task collections---where group routing approaches full task identification \citep{ge2023lightweight}---and extending the medical evaluation to a real multi-site federated setup would broaden the deployment story.

\bibliographystyle{plainnat}
\bibliography{references}

\newpage
\appendix
\section{Broader Impact}
\label{app:broader_impact}

GRAM provides a way to combine independently fine-tuned task vectors into a single deployable model without retraining or pooling raw training data, and the geometry-aware grouping step is what makes this scale beyond a handful of tasks. We highlight two settings where this property carries concrete societal value.

\paragraph{Healthcare and other privacy-constrained domains.}
Clinical AI is trained on data that cannot legally or ethically be centralized: electronic health records, radiology images, pathology slides, and clinical-note adapters live behind institutional firewalls and HIPAA/GDPR boundaries. Today, a hospital that fine-tunes a foundation model on its own patient cohort cannot easily benefit from another hospital's adapter for a related task (a different imaging modality, a different patient population, or a different downstream label) without either sharing the underlying data or running an expensive joint training run. GRAM operates entirely on \emph{released task vectors}---the parameter deltas, never the data---and merges them by clustering subspaces that are geometrically compatible. This lets a coalition of institutions pool model expertise across complementary specialties (e.g., radiology + pathology + clinical-note triage) while keeping each institution's data on-premises, and without forcing dissimilar tasks into a single averaged model where high-stakes minority tasks would otherwise be drowned out by majority gradients. The same property applies to other privacy-sensitive verticals: legal, financial, and on-device personalization.

\paragraph{Decentralized and community-driven model merging.}
The open-model ecosystem now hosts hundreds of thousands of community-released LoRA adapters and full fine-tunes per base model, but practitioners have no principled way to combine more than a few at a time. Naive averaging across many adapters destroys task-specific structure; per-adapter routing keeps a linear deployment cost. GRAM's coalition structure is a middle path: a small number $\nK$ of merged checkpoints that respect subspace geometry, with the compression budget $\nK$ left as an explicit knob. This matches the way decentralized model development actually proceeds---asynchronous, uncoordinated, contributors who never share data---and gives downstream users a way to consume many adapters as a few well-chosen merged models, lowering the deployment cost of community fine-tuning without forcing centralization.

\paragraph{Risks and limitations.}
Two risks are worth flagging. First, merging task vectors does not erase the biases or harmful behaviors of any constituent model; coalitions of biased adapters can yield biased merged checkpoints, and the chordal-distance grouping is geometric, not semantic, so it offers no protection against value misalignment in the inputs. Practitioners should audit each merged coalition with the same fairness and safety scrutiny they would apply to a single fine-tune. Second, GRAM lowers the cost of redistributing capabilities encoded in fine-tuned adapters, including capabilities the original fine-tuner did not intend to release. Standard adapter-licensing and provenance practices---adapter cards, licensing terms, watermarking---apply unchanged, and we encourage their use when GRAM is deployed in shared settings.

\section{Deferred Theoretical Analysis}
\label{app:theory}

We collect the geometric and complexity results that motivate GRAM but are deferred from the main text for space. \Cref{thm:grouping} below is the headline guarantee that grouping reduces total projection error (referenced from \cref{sec:ablation}); the results that follow extend the picture to the GRAM column-stacked basis, the partition-level error chain that ties the chordal clustering rule to a $k$-means++ approximation guarantee, the computational cost, and a Rademacher-complexity bound. \Cref{eq:principal_angle} in \cref{sec:preliminaries} already establishes that, for any orthonormal $B\in\R^{\dout\times r}$, $\fnorm{(I-BB^\top)U_i}^2 = \sum_{j=1}^{r}\sin^2\theta_j$ --- the squared chordal distance between $\mathrm{col}(U_i)$ and $\mathrm{col}(B)$. All proofs are in \cref{app:proofs}.

\begin{theorem}[Grouping reduces total error]
\label{thm:grouping}
Let $\nT$ tasks be partitioned into $\nK$ groups $\{G_1,\ldots,G_\nK\}$. For each group $G_k$, let $B^{(k)} \in \R^{\dout \times r}$ be the optimal rank-$r$ basis minimizing $\sum_{i \in G_k}\fnorm{(I - BB^\top)\dW_i}^2$ over orthonormal $B$, and let $B^* \in \R^{\dout \times r}$ be the analogous global optimum for all $\nT$ tasks. Then
\begin{equation}
\label{eq:group-bound}
\sum_{k=1}^{\nK}\sum_{i \in G_k}
\fnorm{(I - B^{(k)}B^{(k)\top})\dW_i}^2
\;\leq\;
\sum_{i=1}^{\nT}
\fnorm{(I - B^*B^{*\top})\dW_i}^2,
\end{equation}
with strict inequality whenever some per-group optimal basis $B^{(k)}$ improves on the global optimum $B^*$ within group $G_k$. The proof (\cref{app:grouping_proof}) is a direct optimality argument and does not depend on the specific clustering rule.
\end{theorem}

\begin{theorem}[Direct projection-error bound for the GRAM per-group basis]
\label{thm:gram_basis}
Let $G_k = \{1,\ldots,m\}$ be a coalition with rank allocation $(k_1,\ldots,k_m)$ satisfying $0 \le k_i \le r$ and $\sum_{i=1}^{m}k_i \le r$ (the budget enforced by the GRAM allocator of \cref{sec:group_merge}), and let
$$M^{(k)} \;:=\; \bigl[\,U_1^{(k_1)}\;\big|\;\cdots\;\big|\;U_m^{(k_m)}\,\bigr] \;\in\; \R^{\dout\times R}, \qquad R=\sum_{i}k_i \le r,$$
be the column-wise concatenation of the per-task top-$k_i$ left singular vectors of $\dW_i$, with the convention that $U_i^{(0)}$ is an empty matrix (no contribution). Let $\widehat U^{(k)} \in \R^{\dout\times r}$ be any orthonormal basis whose column space contains $\mathrm{col}(M^{(k)})$ (e.g.\ the GRAM merge basis, padded if $R<r$). Then for every task $i \in G_k$,
\begin{equation}
\label{eq:gram_basis_bound}
\bigl\|(I - \widehat U^{(k)}\widehat U^{(k)\top})\,\dW_i\bigr\|_F^2 \;\le\; \sum_{\ell=k_i+1}^{r}\sigma_{i,\ell}^2,
\end{equation}
where the right-hand side reduces to $\fnorm{\dW_i}^2$ when $k_i=0$.
\end{theorem}

\subsection{Grassmannian $k$-Means++ Error Bound}
\label{sec:alignment_penalty}

\begin{lemma}[Alignment penalty for two equally weighted subspaces]
\label{lem:alignment_penalty}
Let $U,V \in \R^{\dout\times r}$ have orthonormal columns and principal angles $0 \le \theta_1 \le \cdots \le \theta_r \le \tfrac{\pi}{2}$. Fix $\sigma > 0$. Then for every matrix $B \in \R^{\dout \times r}$ with $B^\top B = I_r$,
\begin{equation}
\label{eq:alignment_penalty}
\bigl\|(I-BB^\top)\sigma U\bigr\|_F^2 + \bigl\|(I-BB^\top)\sigma V\bigr\|_F^2
\;\ge\; \sigma^2 \sum_{j=1}^{r}\bigl(1 - \cos\theta_j\bigr),
\end{equation}
with equality when the columns of $B$ are the top-$r$ left singular vectors of $[\,U \mid V\,]$.
\end{lemma}

\paragraph{Deterministic latent-group model.}
There are $G \ge 2$ latent groups with prototype subspaces $Q_1,\ldots,Q_G \in \R^{\dout\times r}$, each with orthonormal columns, satisfying $\theta_j(Q_g, Q_h) \ge \phi_{\min} > 0$ for all $g\neq h$ and $j=1,\ldots,r$. There are $\nT = mG$ tasks with $m \ge 2$ per group; task $i$ in group $g(i)$ has $\Sigma_i = \sigma I_r$ and $d_c^2(U_i, Q_{g(i)}) \le \epsilon$. For a partition $\pi=\{G_1,\ldots,G_\nK\}$, define the \emph{optimal-basis merge error}
\begin{equation}
\label{eq:err_def}
  \mathrm{Err}(\pi) \;:=\; \sum_{k=1}^{\nK}\min_{B\in\R^{\dout\times r},\,B^\top B = I_r}\sum_{i\in G_k}\bigl\|(I-BB^\top)\dW_i\bigr\|_F^2.
\end{equation}
Under $\Sigma_i = \sigma I_r$, the optimal per-group basis is the chordal Fr\'{e}chet mean on $\mathrm{Gr}(r,\dout)$ (\cref{prop:frechet}), and $\mathrm{Err}(\pi) = \sigma^2\Phi(\pi)$ with the Grassmannian $k$-means cost $\Phi(\pi) := \sum_{k}\min_{C\in\mathrm{Gr}(r,\dout)}\sum_{i\in G_k}d_c^2(U_i, C)$. The bounds below compare partitions through the projection error of the \emph{optimal} per-group basis ($\mathrm{Err}(\pi)$); they should be read as guarantees on the partition objective and how grouping shapes it, not as guarantees on GRAM's actual column-stacked basis. \Cref{thm:gram_basis} provides the algorithmic-side bound on the GRAM basis itself, and \cref{app:frechet_proof} discusses when the column-stacked basis coincides with the Fr\'{e}chet mean and when it merely upper-bounds it. Let $\pi_{++}$ denote the partition from $k$-means++ on $\{P_i = U_iU_i^\top\}$, $\pi_{\mathrm{rand}}$ a uniformly random equipartition into $G$ groups of size $m$, $\pi^* := \argmin_{\pi}\Phi(\pi)$, and $\pi_{\mathrm{global}}$ the trivial single-block partition $(\nK=1)$.

\begin{theorem}[$k$-means++ partition error bounds]
\label{thm:expected_gap}
Under the deterministic latent-group model with $\nK = G$:
\begin{enumerate}[nosep,leftmargin=*]
  \item \emph{($k$-means++ upper bound on the optimal-basis error.)}
  \begin{equation}\label{eq:upper}
    \E\bigl[\mathrm{Err}(\pi_{++})\bigr] \;\le\; 8(\ln \nK + 2)\cdot\mathrm{Err}(\pi^*) \;\le\; 8(\ln\nK+2)\,\sigma^2\nT\epsilon.
  \end{equation}
  \item \emph{(Random partition lower bound, valid for $\epsilon\le r(1-\cos\phi_{\min})/4$.)}
  \begin{equation}\label{eq:lower}
    \E\bigl[\mathrm{Err}(\pi_{\mathrm{rand}})\bigr] \;\ge\; G\!\left(1-\frac{G}{\binom{\nT}{m}}\right)\sigma^2\bigl[r(1-\cos\phi_{\min}) - 4\sqrt{r\epsilon}\,\bigr].
  \end{equation}
  \item \emph{(Global merge lower bound, valid for $\epsilon\le r(1-\cos\phi_{\min})/4$.)}
  \begin{equation}\label{eq:global_lb}
    \mathrm{Err}(\pi_{\mathrm{global}}) \;\ge\; \lfloor \nT/2\rfloor\,\sigma^2\bigl[r(1-\cos\phi_{\min}) - 4\sqrt{r\epsilon}\,\bigr].
  \end{equation}
\end{enumerate}
\end{theorem}

\begin{corollary}[Projection error chain]
\label{cor:err_chain}
Under the deterministic latent-group model with $\nK=G$, in the regime $\epsilon \le c\,r(1-\cos\phi_{\min})^2/(m\ln \nK)$ for a small absolute constant $c>0$,
\begin{equation}\label{eq:err_chain}
  \E\bigl[\mathrm{Err}(\pi_{++})\bigr] \;\le\; \E\bigl[\mathrm{Err}(\pi_{\mathrm{rand}})\bigr] \;\le\; \mathrm{Err}(\pi_{\mathrm{global}}).
\end{equation}
The right inequality is deterministic and follows from groupwise optimality (\cref{thm:grouping}); the left inequality follows from comparing parts~(1) and~(2) of \cref{thm:expected_gap}, with the regime chosen so that the random lower bound exceeds the GRAM upper bound.
\end{corollary}

\Cref{cor:err_chain} captures the headline message: refining the trivial single-block merge into $\nK$ groups via $k$-means++ on subspace similarities is never worse than a random equipartition (in expectation, in the small-$\epsilon$ regime), and an arbitrary equipartition is in turn never worse than not grouping at all (deterministically). A non-uniform singular-value extension is in \cref{app:nonuniform_proof}.

\subsection{Computational Complexity}
\label{sec:complexity}

Let $d = \max(\dout, \din)$. The GRAM pipeline incurs, once per model:
\begin{itemize}[nosep,leftmargin=*]
  \item \emph{Per-task SVD:} For LoRA deltas $\dW_i = B_i A_i$ (Flan-T5/GLUE, Qwen medical), thin QR of $B_i$ and $A_i^\top$ followed by an $r\times r$ SVD costs $O((\dout+\din) r^2 + r^3)$ per task per layer. For full-rank deltas (vision), a top-$r$ truncated SVD costs $O(\dout \din r)$ per task per layer.
  \item \emph{Pairwise chordal distance:} $d_c^2(U_i,U_j) = r - \fnorm{U_i^\top U_j}^2$ is $O(\dout r^2)$ per pair; the full $\nT\times\nT$ matrix is $O(\nT^2 \dout r^2)$ --- no need to materialize the projection matrices $P_i$.
  \item \emph{Clustering:} $k$-means++ on the precomputed pairwise-distance matrix; negligible when $\nT$ is small.
  \item \emph{Per-group merge:} $O(\nK \nL \dout r^2)$ for the polar-factor orthogonalization plus $O(|G_k| r \log r)$ per layer for rank allocation.
\end{itemize}
The dominant term is $O(\nT\nL(\dout+\din)r^2)$ for LoRA deltas and $O(\nT\nL\dout\din r)$ for full-rank deltas, both strictly cheaper than one epoch of fine-tuning ($O(n_{\text{train}} \nL d^2)$ per task). At inference, each group's rank-$r$ merged update costs $O(dr)$ per layer---identical to individual or global merge on a per-query basis. GRAM stores $\nK$ merged updates instead of $\nT$ per-task models, a factor $\nT/\nK$ reduction in storage with no additional per-query FLOPs.

\subsection{Generalization via Rademacher Complexity}
\label{sec:rademacher}

\begin{definition}[Linear-update hypothesis class]
\label{def:hyp_class}
Fix $W_0 \in \R^{\dout \times \din}$, $a \in \R^{\dout}$ with $\|a\|_2 = 1$, and $\beta > 0$. Let
$\mathcal{H}_a(\beta) = \bigl\{x \mapsto a^\top (W_0 + \dW)\,x : \fnorm{\dW} \le \beta \bigr\}.$
\end{definition}

\begin{theorem}[Rademacher bound for linear updates]
\label{thm:rademacher}
Let $\{x_1, \dots, x_n\}$ be an i.i.d.\ sample with $\|x_i\|_2 \le R$. Then
\begin{equation}\label{eq:rademacher}
  \hat{\mathfrak{R}}_n\!\bigl(\mathcal{H}_a(\beta)\bigr) \;\le\; \beta R/\sqrt{n}.
\end{equation}
\end{theorem}

\begin{corollary}[Per-group capacity is no larger than the global class]
\label{cor:grouping_complexity}
If the GRAM per-group merged updates $\{\dWt^{(k)}\}_{k=1}^{\nK}$ all satisfy $\fnorm{\dWt^{(k)}}\leq\beta$, each per-group update belongs to $\mathcal{H}_a(\beta)$, so the per-group Rademacher complexity is bounded by $\beta R/\sqrt n$, matching the global class.
\end{corollary}

\section{Proof Details}
\label{app:proofs}

\subsection{Proof of \cref{thm:grouping}}
\label{app:grouping_proof}

\begin{proof}[Proof of \cref{thm:grouping}]
For the global merge, let $B^* \in \R^{\dout \times r}$ be the optimal rank-$r$ basis minimizing $\sum_{i=1}^{\nT} \fnorm{(I - B^* B^{*\top})\dW_i}^2$ over all tasks.

For a partition $\{G_1, \ldots, G_\nK\}$, let $B^{(k)}$ be the optimal rank-$r$ basis for group $G_k$, i.e.,
\begin{equation}
    B^{(k)} = \argmin_{B \in \R^{\dout \times r},\, B^\top B = I}
    \sum_{i \in G_k} \fnorm{(I - BB^\top)\dW_i}^2.
\end{equation}

By optimality of $B^{(k)}$ within group $G_k$, for any orthonormal $B \in \R^{\dout \times r}$:
\begin{equation}
    \sum_{i \in G_k} \fnorm{(I - B^{(k)} B^{(k)\top})\dW_i}^2
    \leq \sum_{i \in G_k} \fnorm{(I - BB^\top)\dW_i}^2.
\end{equation}

In particular, choosing $B = B^*$ (the global optimum) for each group:
\begin{equation}
    \sum_{k=1}^{\nK} \sum_{i \in G_k} \fnorm{(I - B^{(k)} B^{(k)\top})\dW_i}^2
    \leq \sum_{k=1}^{\nK} \sum_{i \in G_k} \fnorm{(I - B^* B^{*\top})\dW_i}^2
    = \sum_{i=1}^{\nT} \fnorm{(I - B^* B^{*\top})\dW_i}^2,
\end{equation}
which establishes that group-wise merging with per-group optimal bases always achieves error no greater than global merging.

Strictness of the inequality requires a uniqueness/non-degeneracy assumption: if some $B^{(k)}$ strictly improves on $B^*$ within group $G_k$ (which generically holds when the within-group leading singular subspace differs from the global one), the inequality is strict; otherwise it can be tight.
\end{proof}

\subsection{Proof of \cref{thm:gram_basis}}
\label{app:gram_basis_proof}

We restate the theorem for convenience and then prove it via three elementary lemmas.

\begin{lemma}[Pythagorean identity for orthogonal projectors]
\label{lem:pythag}
For any orthogonal projector $P\in\R^{\dout\times\dout}$ ($P^2=P=P^\top$) and any matrix $A\in\R^{\dout\times\din}$,
$\fnorm{(I-P)A}^2 = \fnorm{A}^2 - \fnorm{PA}^2$.
\end{lemma}
\begin{proof}
$\fnorm{A}^2 = \fnorm{PA + (I-P)A}^2 = \fnorm{PA}^2 + 2\langle PA,(I-P)A\rangle_F + \fnorm{(I-P)A}^2$. The cross term vanishes because $\langle PA,(I-P)A\rangle_F = \tr(A^\top P^\top(I-P)A) = \tr(A^\top P(I-P)A) = 0$, using $P=P^\top$ and $P(I-P)=0$.
\end{proof}

\begin{lemma}[Subspace monotonicity of projected energy]
\label{lem:subspace_mono}
Let $P,Q\in\R^{\dout\times\dout}$ be orthogonal projectors with $\mathrm{col}(Q)\subseteq\mathrm{col}(P)$. Then $P\succeq Q$ in the Loewner order, and for every $A\in\R^{\dout\times\din}$,
$\fnorm{PA}^2 \ge \fnorm{QA}^2$.
\end{lemma}
\begin{proof}
For orthogonal projectors, $\mathrm{col}(Q)\subseteq\mathrm{col}(P)$ is equivalent to $PQ=Q$ (since $Q$ projects into a subspace of $\mathrm{col}(P)$ and $P$ acts as the identity there). Taking transposes and using $P=P^\top$, $Q=Q^\top$ also gives $QP=Q$, so $P-Q$ is symmetric and idempotent: $(P-Q)^2 = P^2 - PQ - QP + Q^2 = P - Q - Q + Q = P-Q$. Hence $P-Q$ is itself an orthogonal projector, and for any $x$,
$x^\top(P-Q)x = \|(P-Q)x\|_2^2 = \|(I-Q)Px\|_2^2 \ge 0$,
where the last equality uses $(P-Q)x = Px - Qx = Px - QPx = (I-Q)Px$. Thus $P\succeq Q$, and applying this with $x = A v$ gives $A^\top(P-Q)A \succeq 0$, so $\tr(A^\top P A) \ge \tr(A^\top Q A)$, i.e.\ $\fnorm{PA}^2 \ge \fnorm{QA}^2$.
\end{proof}

\begin{lemma}[Self-projection recovers the truncation energy]
\label{lem:self_proj}
For any task $i$ with thin SVD $\dW_i = U_i\Sigma_i V_i^\top$ and any $0\le k_i\le r$, letting $P_i := U_i^{(k_i)}U_i^{(k_i)\top}$ (with $P_i = 0$ when $k_i=0$),
$\fnorm{P_i\,\dW_i}^2 = \sum_{\ell=1}^{k_i}\sigma_{i,\ell}^2$, where the sum is empty (and equals $0$) when $k_i=0$.
\end{lemma}
\begin{proof}
The case $k_i=0$ is immediate from $P_i=0$. For $k_i\ge 1$, $U_i^{(k_i)\top}U_i = [\,I_{k_i}\,|\,0_{k_i\times(r-k_i)}\,]$, so
$U_i^{(k_i)\top}\dW_i = [\,I_{k_i}\,|\,0\,]\Sigma_i V_i^\top = [\,\Sigma_i^{(k_i)}\,|\,0\,]V_i^\top = \Sigma_i^{(k_i)}V_i^{(k_i)\top}$,
where $\Sigma_i^{(k_i)}=\diag(\sigma_{i,1},\ldots,\sigma_{i,k_i})$ and $V_i^{(k_i)}\in\R^{\din\times k_i}$. Since $U_i^{(k_i)}$ has orthonormal columns, $\fnorm{P_i\dW_i}^2 = \fnorm{U_i^{(k_i)\top}\dW_i}^2 = \fnorm{\Sigma_i^{(k_i)}V_i^{(k_i)\top}}^2 = \fnorm{\Sigma_i^{(k_i)}}^2 = \sum_{\ell=1}^{k_i}\sigma_{i,\ell}^2$ (using $\fnorm{XY^\top}=\fnorm{X}$ when $Y$ has orthonormal columns).
\end{proof}

\begin{proof}[Proof of \cref{thm:gram_basis}]
Let $P^{(k)} := \widehat U^{(k)}\widehat U^{(k)\top}$ be the orthogonal projector onto $\mathrm{col}(\widehat U^{(k)})$. By the construction of $\widehat U^{(k)}$, $\mathrm{col}(\widehat U^{(k)}) = \mathrm{col}(M^{(k)})$. (This holds in particular for the orthogonal polar factor of $M^{(k)}$ used in GRAM: if $M^{(k)} = u\,s\,v^\top$ is the SVD, then $u v^\top$ has the same column space as $u$, which equals $\mathrm{col}(M^{(k)})$.)

Since $U_i^{(k_i)}$ is one of the column blocks of $M^{(k)}$ by construction, every column of $U_i^{(k_i)}$ lies in $\mathrm{col}(M^{(k)}) = \mathrm{col}(\widehat U^{(k)})$, so
\begin{equation}
  \mathrm{col}(P_i) \;=\; \mathrm{col}(U_i^{(k_i)}) \;\subseteq\; \mathrm{col}(\widehat U^{(k)}) \;=\; \mathrm{col}(P^{(k)}).
\end{equation}
By \cref{lem:subspace_mono}, $\fnorm{P^{(k)}\dW_i}^2 \ge \fnorm{P_i\dW_i}^2$, and by \cref{lem:self_proj}, $\fnorm{P_i\dW_i}^2 = \sum_{\ell=1}^{k_i}\sigma_{i,\ell}^2$. Combining,
\begin{equation}
  \fnorm{P^{(k)}\dW_i}^2 \;\ge\; \sum_{\ell=1}^{k_i}\sigma_{i,\ell}^2.
\end{equation}
Applying \cref{lem:pythag} to $P=P^{(k)}$ and $A=\dW_i$ gives $\fnorm{(I-P^{(k)})\dW_i}^2 = \fnorm{\dW_i}^2 - \fnorm{P^{(k)}\dW_i}^2$. Substituting the lower bound and using $\fnorm{\dW_i}^2 = \sum_{\ell=1}^r\sigma_{i,\ell}^2$,
\begin{equation}
  \fnorm{(I-P^{(k)})\dW_i}^2 \;\le\; \sum_{\ell=1}^{r}\sigma_{i,\ell}^2 - \sum_{\ell=1}^{k_i}\sigma_{i,\ell}^2 \;=\; \sum_{\ell=k_i+1}^{r}\sigma_{i,\ell}^2,
\end{equation}
which is~\eqref{eq:gram_basis_bound}.

\textbf{Equality case ($m=1$).} If the coalition has a single task, $M^{(k)} = U_i^{(k_i)}$ and $\mathrm{col}(P^{(k)}) = \mathrm{col}(P_i)$, so the inequality in \cref{lem:subspace_mono} is an equality, and the bound is attained.
\end{proof}

\paragraph{What this proof does and does not give.}
The proof bounds the projection error onto the GRAM column basis, not the reconstruction error
$\|\dW_i - \dWt^{(k)}\|_F^2$ of the merged delta itself. The latter decomposes orthogonally as
\begin{equation}
\label{eq:gram_orth_decomp}
  \|\dW_i - \dWt^{(k)}\|_F^2 \;=\; \|(I-P^{(k)})\dW_i\|_F^2 + \|P^{(k)}\dW_i - \dWt^{(k)}\|_F^2,
\end{equation}
and \cref{thm:gram_basis} controls only the first term. The second term --- the in-subspace fitting error of the single merged delta against task $i$'s in-subspace component --- is not analyzed here and would require a separate argument tracking how $\dWt^{(k)} = \widehat U^{(k)}\,\diag(\bar{s})\,\widehat V^{(k)\top}$ is constructed from the per-task singular values and right-singular vectors. The bound in \cref{thm:gram_basis} is also loose precisely when other tasks in the coalition contribute new directions that align with $\mathrm{col}(U_i)\setminus\mathrm{col}(U_i^{(k_i)})$ --- the regime in which grouping is most useful --- because the proof discards all inter-task overlap by lower-bounding $\fnorm{P^{(k)}\dW_i}^2$ via the smaller projector $P_i$ alone.

\subsection{GRAM vs.\ global merge under uniform allocation (statement and proof)}
\label{app:gram_vs_global_proof}

\begin{corollary}[Grouping reduces the projection-error upper bound for every task]
\label{cor:gram_vs_global}
Fix a LoRA rank budget $\rho$ per merged adapter. Compare the global merge over all $\nT$ tasks under uniform allocation $k_i^{\mathrm{glob}} := \lfloor\rho/\nT\rfloor$ to a GRAM coalition $G_k$ of size $m < \nT$ under uniform allocation $k_i := \lfloor\rho/m\rfloor$. Then $k_i^{\mathrm{glob}} \le k_i$, and for every task $i \in G_k$ the projection-error upper bounds from \cref{thm:gram_basis} satisfy
\begin{equation}
\label{eq:gram_vs_global}
\underbrace{\sum_{\ell=k_i+1}^{r}\sigma_{i,\ell}^2}_{\text{GRAM bound}}
\;\le\;
\underbrace{\sum_{\ell=k_i^{\mathrm{glob}}+1}^{r}\sigma_{i,\ell}^2}_{\text{global bound}},
\qquad
\text{improvement} \;\;\ge\;\; \sum_{\ell=k_i^{\mathrm{glob}}+1}^{k_i}\sigma_{i,\ell}^2 \;\ge\; 0.
\end{equation}
The improvement is strictly positive whenever any of the singular values $\sigma_{i,k_i^{\mathrm{glob}}+1},\ldots,\sigma_{i,k_i}$ is non-zero.
\end{corollary}

\begin{proof}
$\nT > m$ implies $\lfloor\rho/\nT\rfloor \le \lfloor\rho/m\rfloor$, hence $k_i^{\mathrm{glob}} \le k_i$. The tail sum $\sum_{\ell=k+1}^{r}\sigma_{i,\ell}^2$ is non-increasing in $k$, so the GRAM bound is no larger than the global bound; the difference is exactly $\sum_{\ell=k_i^{\mathrm{glob}}+1}^{k_i}\sigma_{i,\ell}^2$. Applying \cref{thm:gram_basis} to both bases yields~\eqref{eq:gram_vs_global}.
\end{proof}

\begin{remark}[Tightness]
\Cref{cor:gram_vs_global} states a worst-case guarantee on the gap between GRAM and global merging at the level of upper bounds, not a comparison of actual projection errors; the looseness of \cref{thm:gram_basis} discussed above carries through here. The multiplicative gap between alignment-aware and random grouping is captured by \cref{thm:expected_gap}, not this corollary.
\end{remark}

\subsection{Equivalence of chordal and Euclidean distances on projectors}
\label{app:chordal_eucl_proof}

To run $k$-means++ off the shelf we need a Euclidean embedding of the Grassmannian. Each subspace $\mathrm{col}(U_i^{(\ell)})$ has a unique \emph{orthogonal projector} $P_i^{(\ell)} := U_i^{(\ell)}U_i^{(\ell)\top} \in \R^{\dout\times\dout}$, which is invariant to the choice of orthonormal basis for the subspace and serves as that embedding. On these projectors,
\begin{equation}
\label{eq:chordal_eucl}
\fnorm{P_i^{(\ell)} - P_j^{(\ell)}}^2 \;=\; \fnorm{P_i^{(\ell)}}^2 + \fnorm{P_j^{(\ell)}}^2 - 2\,\langle P_i^{(\ell)}, P_j^{(\ell)}\rangle \;=\; 2\,d_c^2\!\bigl(U_i^{(\ell)},\, U_j^{(\ell)}\bigr),
\end{equation}
using $\fnorm{P_i^{(\ell)}}^2 = r$ (rank-$r$ idempotent) and $\langle P_i^{(\ell)}, P_j^{(\ell)}\rangle = \fnorm{U_i^{(\ell)\top}U_j^{(\ell)}}^2$. The chordal distance on the Grassmannian is therefore Euclidean (up to a factor of $2$) on the projector embedding.

\subsection{Chordal Fr\'{e}chet mean and the GRAM merge basis}
\label{app:frechet_proof}

\begin{proposition}[Fr\'{e}chet mean under chordal distance]
\label{prop:frechet}
Let $\{U_1,\ldots,U_m\}\subset\R^{\dout\times r}$ have orthonormal columns and let $M := \sum_{i=1}^{m}U_iU_i^\top$. Any minimizer of $\sum_{i=1}^{m}d_c^2(U_i, C)$ over $C\in\mathrm{Gr}(r,\dout)$ is a top-$r$ invariant subspace of $M$; the minimizer is unique whenever $\lambda_r(M) > \lambda_{r+1}(M)$, and otherwise any choice of $r$-dimensional subspace within the top eigenspace is optimal.
\end{proposition}

\begin{proof}
Write $d_c^2(U_i,C) = r - \fnorm{C^\top U_i}^2 = r - \tr(C^\top U_iU_i^\top C)$ for any orthonormal $C\in\R^{\dout\times r}$. Therefore
$$\sum_{i=1}^{m}d_c^2(U_i, C) \;=\; mr - \tr\!\bigl(C^\top M C\bigr).$$
Minimizing over orthonormal $C$ is equivalent to maximizing $\tr(C^\top M C)$ where $M\succeq 0$. By the Ky Fan max-trace theorem, the maximum is $\sum_{j=1}^{r}\lambda_j(M)$ and is achieved when $C$ spans the top-$r$ eigenspace of $M$.
\end{proof}

\paragraph{Relation to the GRAM column-stacked basis.}
\Cref{thm:gram_basis} concerns the column space of $M^{(k)} = [U_1^{(k_1)}\!\mid\!\cdots\!\mid\!U_m^{(k_m)}]$, which has dimension at most $R=\sum_i k_i$ (and is set to $r$ by the GRAM allocation rule of \cref{sec:group_merge}, which exhausts the rank budget). The Fr\'{e}chet-mean basis above is the top-$r$ eigenspace of $\sum_i U_iU_i^\top$, and the two coincide only when $\mathrm{col}(M^{(k)})$ equals that eigenspace; in general they differ. For example, with $m=r=2$, $\Sigma_i=\sigma I_2$, $U_1=U_2=[e_1,e_2]$, and allocation $k_1=k_2=1$, the selected column $e_1$ from each task gives $\mathrm{col}(M^{(k)})=\mathrm{span}(e_1)$ while the Fr\'{e}chet mean is $\mathrm{span}(e_1,e_2)$. We therefore distinguish two regimes in the analysis below:
(i) the per-task projection-error \emph{upper} bound \cref{eq:gram_basis_bound}, which holds for \emph{any} orthonormal basis of $\mathrm{col}(M^{(k)})$ — including the GRAM output — without invoking the Fr\'{e}chet-mean identification; and
(ii) the latent-group analysis ($k$-means++ gap, error chain), in which $\mathrm{Err}(\pi)$ is defined via the optimal-basis projection error and the Fr\'{e}chet mean is invoked as the optimum, not as GRAM's actual output. This separation makes the theory bridge orthogonal to whether GRAM's column-stacked basis exactly matches the Fr\'{e}chet mean.

\subsection{Proof of \cref{lem:alignment_penalty}}
\label{app:alignment_penalty_proof}

\begin{proof}[Proof of \cref{lem:alignment_penalty}]
Let $P:=BB^\top$ and $S := UU^\top + VV^\top$. Using $\fnorm{(I-P)\sigma X}^2 = \sigma^2\bigl(\fnorm{X}^2 - \fnorm{B^\top X}^2\bigr)$ with $\fnorm{U}^2=\fnorm{V}^2=r$ and $\fnorm{B^\top X}^2 = \tr(B^\top XX^\top B)$,
\begin{equation}\label{eq:align_step1}
  \fnorm{(I-P)\sigma U}^2 + \fnorm{(I-P)\sigma V}^2 \;=\; \sigma^2\bigl(2r - \tr(B^\top S\,B)\bigr).
\end{equation}
Since $S = MM^\top$ with $M := [\,U\mid V\,]\in\R^{\dout\times 2r}$, the nonzero eigenvalues of $S$ equal the nonzero eigenvalues of
$$M^\top M \;=\; \begin{pmatrix} I_r & U^\top V \\ V^\top U & I_r \end{pmatrix}.$$
The singular values of $U^\top V$ are $\cos\theta_1,\ldots,\cos\theta_r$ (Golub--Van Loan), so $M^\top M$ is unitarily similar to $\bigoplus_{j=1}^r\bigl(\begin{smallmatrix}1 & \cos\theta_j\\ \cos\theta_j & 1\end{smallmatrix}\bigr)$ and its eigenvalues are $\{1+\cos\theta_j,\,1-\cos\theta_j\}_{j=1}^{r}$. Since $\cos\theta_j\ge 0$, the top-$r$ eigenvalues of $S$ are $\{1+\cos\theta_j\}_{j=1}^r$. By the Ky Fan max-trace theorem,
\begin{equation}\label{eq:align_step2}
  \tr(B^\top S\,B) \;\le\; \sum_{j=1}^{r}(1+\cos\theta_j),
\end{equation}
with equality when $\mathrm{col}(B)$ is a top-$r$ eigenspace of $S$, equivalently the top-$r$ left singular subspace of $[\,U\mid V\,]$. Substituting \eqref{eq:align_step2} into \eqref{eq:align_step1} gives
$$\fnorm{(I-P)\sigma U}^2 + \fnorm{(I-P)\sigma V}^2 \;\ge\; \sigma^2\Bigl(2r - \sum_{j=1}^{r}(1+\cos\theta_j)\Bigr) \;=\; \sigma^2\sum_{j=1}^{r}(1-\cos\theta_j),$$
which is~\eqref{eq:alignment_penalty}.
\end{proof}

\subsection{Proof of \cref{thm:expected_gap}}
\label{app:kmeans_proof}

We prove each part of the theorem under the deterministic latent-group model: $G$ groups, $m$ tasks per group ($\nT=mG$), prototype subspaces $Q_1,\ldots,Q_G$ with pairwise principal angles $\ge\phi_{\min}$, and within-group spread $d_c^2(U_i,Q_{g(i)})\le\epsilon$ for all $i$.

\begin{proof}[Proof of part (1): upper bound]
Define $P_i = U_iU_i^\top$ and recall $\fnorm{P_i-P_j}^2 = 2\,d_c^2(U_i,U_j)$.

\emph{Step~1 (seeding guarantee).}
$k$-means++ selects $\nK$ seeds $s_1,\ldots,s_\nK$ from $\{P_i\}$ via $D^2$-sampling and assigns each $P_i$ to its nearest seed, yielding partition $\pi_{++}$. By~\citet{arthur2007kmeans},
\begin{equation}\label{eq:kmeans_guarantee_app}
\E\!\left[\sum_{k=1}^{\nK}\sum_{i\in G_k^{++}}\fnorm{P_i - s_k}^2\right] \;\le\; 8(\ln\nK+2)\cdot\mathrm{OPT}_E,
\end{equation}
where $\mathrm{OPT}_E := \min_{\pi,\{\mu_k\}}\sum_k\sum_{i\in G_k}\fnorm{P_i-\mu_k}^2$.

\emph{Step~2 (seeds are projections).}
Each seed $s_k = P_{j_k}$ is a rank-$r$ projection matrix, hence a valid point on $\mathrm{Gr}(r,\dout)$. The chordal Fr\'{e}chet mean minimizes cost over rank-$r$ projections and can only reduce cost:
$$\Phi(\pi_{++}) \;\le\; \sum_{k}\sum_{i\in G_k^{++}}d_c^2(U_i, U_{j_k}) \;=\; \tfrac{1}{2}\sum_{k}\sum_{i\in G_k^{++}}\fnorm{P_i - s_k}^2.$$

\emph{Step~3 (Euclidean $\le$ Grassmannian).}
For any partition $\pi$ and group $G_k$, the Euclidean mean $\bar{P}_k$ minimizes $\sum_i\fnorm{P_i-\mu}^2$ over all matrices $\mu$, while the Fr\'{e}chet mean $P_{C_k}$ is constrained to rank-$r$ projections. Therefore $\sum_i\fnorm{P_i-\bar{P}_k}^2 \le \sum_i\fnorm{P_i-P_{C_k}}^2 = 2\sum_i d_c^2(U_i,C_k)$. Minimizing over $C_k$ and summing: $\mathrm{OPT}_E \le 2\,\Phi(\pi^*)$.

\emph{Combining.} $\E[\Phi(\pi_{++})] \le \frac{1}{2}\cdot 8(\ln\nK+2)\cdot 2\,\Phi(\pi^*) = 8(\ln\nK+2)\,\Phi(\pi^*)$. Multiplying by $\sigma^2$ gives $\E[\mathrm{Err}(\pi_{++})] \le 8(\ln\nK+2)\,\mathrm{Err}(\pi^*)$.

\emph{Absolute bound.} The latent partition $\pi_{\mathrm{lat}}$ is feasible:
$\Phi(\pi^*) \le \Phi(\pi_{\mathrm{lat}}) = \sum_{g=1}^G\sum_{i:g(i)=g}\!d_c^2(U_i,C_g^{\mathrm{lat}}) \le \sum_{i=1}^{\nT}d_c^2(U_i,Q_{g(i)}) \le \nT\epsilon,$
where the second inequality uses $d_c^2(U_i,C_g^{\mathrm{lat}}) \le d_c^2(U_i,Q_{g(i)})$ since $C_g^{\mathrm{lat}}$ is the cost-minimizing center.
\end{proof}

\begin{proof}[Proof of part (2): lower bound]
Call $G_k$ \emph{mixed} if it contains tasks $i,j$ with $g(i)\neq g(j)$, and pick any such pair $(i,j)$ in each mixed $G_k$. Let $C_k\in\R^{\dout\times r}$ be the chordal Fr\'{e}chet mean of $G_k$, which is the minimizer in \eqref{eq:err_def} for $\Sigma_l=\sigma I_r$ by \cref{prop:frechet}. \Cref{lem:alignment_penalty} holds for \emph{every} orthonormal $B$, so applying it with $B=C_k$ to the pair $(U_i,U_j)$ and dropping the non-negative contributions from the remaining $|G_k|-2$ tasks,
\begin{equation}\label{eq:per_group_lb}
  \sum_{l\in G_k}\bigl\|(I-C_kC_k^\top)\dW_l\bigr\|_F^2 \;\ge\; \sigma^2\sum_{s=1}^{r}\bigl(1-\cos\theta_s(U_i,U_j)\bigr).
\end{equation}
The left side is exactly the contribution of $G_k$ to $\mathrm{Err}(\pi)$ (since $C_k$ achieves the minimum in \eqref{eq:err_def}), so this is a per-mixed-group lower bound on $\mathrm{Err}(\pi)$.

\emph{Procrustes metric and triangle inequality.}
Define $\delta_p(U,V) := \bigl[2\sum_{s=1}^{r}\bigl(1-\cos\theta_s(U,V)\bigr)\bigr]^{1/2} = \min_{R\in O(r)}\fnorm{U-VR}$, the Procrustes (chordal Fubini--Study) distance, which is a metric on $\mathrm{Gr}(r,\dout)$ \citep{edelman1998geometry}; $\delta_p^2$ denotes its square. We use two facts:
\emph{(a)} since $1-\cos t \le \sin^2 t$ for $t\in[0,\pi/2]$, $\delta_p^2(U,V)\le 2\sum_s\sin^2\theta_s = 2 d_c^2(U,V)$; and
\emph{(b)} since $1-\cos t \ge (1-\cos t)\cdot 1$ trivially with $\theta_s(Q_g,Q_h)\ge\phi_{\min}$, $\delta_p^2(Q_g,Q_h)\ge 2r(1-\cos\phi_{\min})$.
Combining (a) with $d_c^2(U_i,Q_g)\le\epsilon$ gives $\delta_p(U_i,Q_g)\le\sqrt{2\epsilon}$, and likewise for $j$. By the triangle inequality for $\delta_p$,
\begin{equation}\label{eq:proc_triangle}
  \delta_p(U_i,U_j) \;\ge\; \delta_p(Q_g,Q_h) - \delta_p(U_i,Q_g) - \delta_p(U_j,Q_h) \;\ge\; \sqrt{2r(1-\cos\phi_{\min})} - 2\sqrt{2\epsilon}.
\end{equation}
The right side is non-negative iff $\epsilon\le r(1-\cos\phi_{\min})/4$, which is the hypothesis of part~(2). Squaring and dropping the non-negative cross term,
\begin{align*}
  \delta_p^2(U_i,U_j) &\;\ge\; 2r(1-\cos\phi_{\min}) - 2\cdot 2\sqrt{2}\sqrt{r(1-\cos\phi_{\min})}\cdot\sqrt{2\epsilon}\\
  &\;=\; 2r(1-\cos\phi_{\min}) - 8\sqrt{r\epsilon(1-\cos\phi_{\min})}\\
  &\;\ge\; 2r(1-\cos\phi_{\min}) - 8\sqrt{r\epsilon},
\end{align*}
using $1-\cos\phi_{\min}\le 1$. Halving, $\sum_{s=1}^{r}(1-\cos\theta_s(U_i,U_j))\ge r(1-\cos\phi_{\min}) - 4\sqrt{r\epsilon}$. Substituting into \eqref{eq:per_group_lb}, the per-mixed-group contribution to $\mathrm{Err}(\pi_{\mathrm{rand}})$ is at least $\sigma^2[r(1-\cos\phi_{\min}) - 4\sqrt{r\epsilon}]$.

\emph{Expected number of mixed groups.}
Under a uniformly random equipartition into $G$ groups of size $m$, each group is a random $m$-subset; exactly $G$ of the $\binom{\nT}{m}$ such subsets are monochromatic (one per latent group). Hence $\E[N_{\mathrm{mix}}] = G\bigl(1-G/\binom{\nT}{m}\bigr)$. Linearity of expectation across mixed groups gives~\eqref{eq:lower}.
\end{proof}

\begin{proof}[Proof of part (3): global merge lower bound]
Let $B^*\in\R^{\dout\times r}$ be the global merge basis ($\pi_{\mathrm{global}}$ has a single block containing all $\nT=mG$ tasks). We construct $\lfloor \nT/2\rfloor$ pairwise-disjoint inter-group pairs and apply \cref{lem:alignment_penalty} to each.

\emph{Existence of the matching.}
With $G\ge 2$ latent groups of size $m$, no group contains more than $\nT/2$ tasks (since $m \le \nT/2$ for $G\ge 2$). A standard graph-theoretic result then gives a matching of size $\lfloor \nT/2\rfloor$ in the inter-group graph (the complete multipartite graph $K_{m,m,\ldots,m}$): order tasks as $\tau_1,\ldots,\tau_\nT$ by interleaving across groups (one task from each group, then a second from each, etc.), and pair $(\tau_{2k-1},\tau_{2k})$ for $k=1,\ldots,\lfloor \nT/2\rfloor$; consecutive tasks in this ordering belong to different groups whenever $G\ge 2$.

\emph{Per-pair lower bound.}
For each disjoint inter-group pair $(i,j)$ with $g(i)\neq g(j)$, applying \cref{lem:alignment_penalty} with $B=B^*$ and the Procrustes argument from the proof of part~(2),
\begin{equation}\label{eq:global_pair_lb}
  \fnorm{(I-B^*B^{*\top})\dW_i}^2 + \fnorm{(I-B^*B^{*\top})\dW_j}^2 \;\ge\; \sigma^2\bigl[r(1-\cos\phi_{\min}) - 4\sqrt{r\epsilon}\bigr],
\end{equation}
valid for $\epsilon\le r(1-\cos\phi_{\min})/4$.

\emph{Summation.}
Summing \eqref{eq:global_pair_lb} over the $\lfloor \nT/2\rfloor$ disjoint pairs and adding the non-negative contribution of any unmatched task (when $\nT$ is odd),
$$\mathrm{Err}(\pi_{\mathrm{global}}) \;=\; \sum_{i=1}^{\nT}\fnorm{(I-B^*B^{*\top})\dW_i}^2 \;\ge\; \lfloor \nT/2\rfloor\,\sigma^2\bigl[r(1-\cos\phi_{\min}) - 4\sqrt{r\epsilon}\bigr],$$
which is~\eqref{eq:global_lb}.
\end{proof}

\subsection{Proof of \cref{cor:err_chain}}
\label{app:err_chain_proof}

\begin{proof}[Proof of \cref{cor:err_chain}]
\emph{Right inequality (deterministic).}
For any partition $\pi$ with groups $\{G_1,\ldots,G_\nK\}$, let $B^{(k)}_{\pi}$ minimize $\sum_{i\in G_k}\fnorm{(I-BB^\top)\dW_i}^2$ over orthonormal $B$, so $\mathrm{Err}(\pi)=\sum_k\sum_{i\in G_k}\fnorm{(I-B^{(k)}_\pi B^{(k)\top}_\pi)\dW_i}^2$. Let $B^*$ be the analogous global minimizer for $\pi_{\mathrm{global}}$ (a single block of all $\nT$ tasks). Substituting $B=B^*$ as a feasible (sub-optimal) basis inside each group $G_k$,
\begin{align*}
  \mathrm{Err}(\pi) \;=\; \sum_{k=1}^{\nK}\sum_{i\in G_k}\fnorm{(I-B^{(k)}_\pi B^{(k)\top}_\pi)\dW_i}^2
  &\;\le\; \sum_{k=1}^{\nK}\sum_{i\in G_k}\fnorm{(I-B^*B^{*\top})\dW_i}^2\\
  &\;=\; \sum_{i=1}^{\nT}\fnorm{(I-B^*B^{*\top})\dW_i}^2 \;=\; \mathrm{Err}(\pi_{\mathrm{global}}).
\end{align*}
This is exactly \cref{thm:grouping}; it is deterministic and holds in particular for $\pi=\pi_{\mathrm{rand}}$, hence $\E[\mathrm{Err}(\pi_{\mathrm{rand}})]\le\mathrm{Err}(\pi_{\mathrm{global}})$.

\emph{Left inequality (in expectation, in the small-$\epsilon$ regime).}
Let $\alpha := 1-\cos\phi_{\min}\in(0,2]$. From \cref{thm:expected_gap} parts~(1) and~(2),
$$\E[\mathrm{Err}(\pi_{++})] \le 8(\ln\nK+2)\,\sigma^2\nT\epsilon, \qquad \E[\mathrm{Err}(\pi_{\mathrm{rand}})] \ge G\bigl(1-\tfrac{G}{\binom{\nT}{m}}\bigr)\sigma^2[r\alpha - 4\sqrt{r\epsilon}],$$
the latter valid for $\epsilon \le r\alpha/4$. Sufficient for the left inequality is $\E[\mathrm{Err}(\pi_{\mathrm{rand}})] \ge \E[\mathrm{Err}(\pi_{++})]$, i.e.\
$$G\bigl(1-\tfrac{G}{\binom{\nT}{m}}\bigr)[r\alpha - 4\sqrt{r\epsilon}] \;\ge\; 8(\ln\nK+2)\,\nT\epsilon.$$
Using $\nT = mG$ and $1-G/\binom{\nT}{m}\ge 1/2$ for $G\ge 2$, $m\ge 2$, this is implied by $r\alpha - 4\sqrt{r\epsilon} \ge 16(\ln\nK+2)\,m\,\epsilon$. The bracketed term is $\ge r\alpha/2$ whenever $\epsilon\le r\alpha^2/64$, and the right side is $\le r\alpha/2$ whenever $\epsilon\le r\alpha/(32(\ln\nK+2)m)$. Combining these (using $\alpha\le 2$, $\nK\ge 2$, $m\ge 2$, $\ln\nK + 2 \le 4\ln\nK$ for $\nK\ge 2$), both inequalities hold whenever $\epsilon\le c\,r\alpha^2/(m\ln\nK)$ for a sufficiently small absolute constant $c>0$. Hence $\E[\mathrm{Err}(\pi_{++})]\le\E[\mathrm{Err}(\pi_{\mathrm{rand}})]$.
\end{proof}

\subsection{Non-uniform singular values (statement and proof)}
\label{app:nonuniform_proof}

\begin{proposition}[Non-uniform singular values]
\label{prop:nonuniform}
Let $\kappa := \sigma_{\max}/\sigma_{\min}$ with $\sigma_{\max} := \max_{i,k}\sigma_{i,k}$, $\sigma_{\min} := \min_{i,k}\sigma_{i,k}$, and recall the optimal-basis merge error \eqref{eq:err_def}. Then for any partition $\pi$,
\begin{equation}\label{eq:err_sandwich}
  \sigma_{\min}^2\,\Phi(\pi) \;\le\; \mathrm{Err}(\pi) \;\le\; \sigma_{\max}^2\,\Phi(\pi),
\end{equation}
and consequently
\begin{equation}\label{eq:kappa_bound}
  \E\bigl[\mathrm{Err}(\pi_{++})\bigr] \;\le\; 8(\ln\nK+2)\,\kappa^2\cdot\mathrm{Err}(\pi^*_w),
\end{equation}
where $\pi^*_w := \argmin_\pi\mathrm{Err}(\pi)$.
\end{proposition}

\begin{proof}
Fix any orthonormal $B\in\R^{\dout\times r}$ and any task $i$. Substituting the SVD $\dW_i = U_i\Sigma_i V_i^\top$ and using $V_i^\top V_i = I_r$,
$$\fnorm{(I-BB^\top)\dW_i}^2 = \tr\!\bigl(\dW_i^\top(I-BB^\top)\dW_i\bigr) = \tr\!\bigl(\Sigma_i^2\,A_i\bigr), \qquad A_i := U_i^\top(I-BB^\top)U_i \;\succeq\; 0,$$
with $\tr(A_i) = r - \fnorm{B^\top U_i}^2 = d_c^2(U_i,B)$ by \eqref{eq:principal_angle}. Since $\sigma_{\min}^2 I_r \preceq \Sigma_i^2 \preceq \sigma_{\max}^2 I_r$ in the Loewner order,
\begin{equation}\label{eq:loewner_sandwich}
  \sigma_{\min}^2\,d_c^2(U_i,B) \;\le\; \fnorm{(I-BB^\top)\dW_i}^2 \;\le\; \sigma_{\max}^2\,d_c^2(U_i,B).
\end{equation}
Now fix a partition $\pi$. For the lower bound, let $B^{(k)}_w := \argmin_{B}\sum_{i\in G_k}\fnorm{(I-BB^\top)\dW_i}^2$ achieve $\mathrm{Err}(\pi)$. By the lower side of \eqref{eq:loewner_sandwich},
$$\mathrm{Err}(\pi) \;=\; \sum_{k}\sum_{i\in G_k}\fnorm{(I-B^{(k)}_w B^{(k)\top}_w)\dW_i}^2 \;\ge\; \sigma_{\min}^2\sum_{k}\sum_{i\in G_k}d_c^2(U_i,B^{(k)}_w) \;\ge\; \sigma_{\min}^2\,\Phi(\pi),$$
where the last inequality uses that $\Phi(\pi)$ is the minimum of $\sum_k\sum_{i\in G_k}d_c^2(U_i, C^{(k)})$ over orthonormal $C^{(k)}$.
For the upper bound, let $C^{(k)}_*$ be the chordal Fr\'{e}chet mean of $G_k$ (achieving $\Phi(\pi)$). Since $C^{(k)}_*$ is a feasible (sub-optimal for the weighted problem) basis,
$$\mathrm{Err}(\pi) \;\le\; \sum_{k}\sum_{i\in G_k}\fnorm{(I-C^{(k)}_*C^{(k)\top}_*)\dW_i}^2 \;\le\; \sigma_{\max}^2\sum_{k}\sum_{i\in G_k}d_c^2(U_i,C^{(k)}_*) \;=\; \sigma_{\max}^2\,\Phi(\pi),$$
using the upper side of \eqref{eq:loewner_sandwich}. This is~\eqref{eq:err_sandwich}.

For \eqref{eq:kappa_bound}, let $\pi^*_w := \argmin_\pi\mathrm{Err}(\pi)$ and let $\pi^* := \argmin_\pi\Phi(\pi)$. Steps 1--3 of the proof of \cref{thm:expected_gap}~(1) are model-independent: they show $\E[\Phi(\pi_{++})]\le 8(\ln\nK+2)\,\Phi(\pi^*)$ for any data $\{U_i\}$, using only the $k$-means++ guarantee \citep{arthur2007kmeans} and the Euclidean--Grassmannian comparison. Combining this with \eqref{eq:err_sandwich},
\begin{align*}
  \E\bigl[\mathrm{Err}(\pi_{++})\bigr] \;\le\; \sigma_{\max}^2\,\E[\Phi(\pi_{++})] &\;\le\; \sigma_{\max}^2\cdot 8(\ln\nK+2)\,\Phi(\pi^*)\\
  &\;\le\; 8(\ln\nK+2)\,\kappa^2\,\sigma_{\min}^2\,\Phi(\pi^*_w) \;\le\; 8(\ln\nK+2)\,\kappa^2\,\mathrm{Err}(\pi^*_w),
\end{align*}
using $\Phi(\pi^*)\le\Phi(\pi^*_w)$ (definition of $\pi^*$) and $\sigma_{\min}^2\Phi(\pi^*_w)\le\mathrm{Err}(\pi^*_w)$ (the lower side of~\eqref{eq:err_sandwich}).
\end{proof}

\subsection{Proof of \cref{thm:rademacher}}
\label{app:rademacher}

\begin{proof}
Let $\epsilon_1,\ldots,\epsilon_n$ be i.i.d.\ Rademacher random variables. Splitting each $h \in \mathcal{H}_a(\beta)$ as $h(x) = a^\top W_0 x + a^\top\dW\, x$, the $W_0$ part is a fixed function of $x$ (independent of $\dW$), so
\begin{equation}
  \frac{1}{n}\sum_i \epsilon_i\,h(x_i) \;=\; \frac{1}{n}\sum_i \epsilon_i\, a^\top W_0 x_i \;+\; \frac{1}{n}\sum_i \epsilon_i\, a^\top \dW\, x_i,
\end{equation}
and the first term factors out of the supremum over $\dW$. Taking expectation over $\epsilon$ and using $\E[\epsilon_i]=0$ to kill the first term,
\begin{equation}
  \hat{\mathfrak{R}}_n(\mathcal{H}_a(\beta))
  \;=\; \E_\epsilon \!\left[\sup_{\fnorm{\dW}\le \beta}\frac{1}{n}\sum_{i=1}^{n} \epsilon_i\, a^\top \dW\, x_i\right].
\end{equation}

\textbf{Step 1: rewrite as a Frobenius inner product.}
For any $\dW$, $a^\top \dW\, x_i = \tr(a^\top \dW x_i) = \langle \dW,\, a x_i^\top\rangle_F$, where $\langle A,B\rangle_F=\tr(A^\top B)$. Linearity gives
\begin{equation}
  \frac{1}{n}\sum_{i=1}^n \epsilon_i\, a^\top \dW\, x_i
  \;=\; \Bigl\langle \dW,\;\frac{1}{n}\sum_{i=1}^n \epsilon_i\, a x_i^\top\Bigr\rangle_F.
\end{equation}

\textbf{Step 2: supremum via Cauchy--Schwarz.}
Let $G_n := \frac{1}{n}\sum_{i=1}^n \epsilon_i\, a x_i^\top \in \R^{\dout\times\din}$. By the Cauchy--Schwarz inequality on the Frobenius inner product,
\begin{equation}
  \sup_{\fnorm{\dW}\le \beta} \langle \dW, G_n\rangle_F \;=\; \beta\,\fnorm{G_n}.
\end{equation}
Moreover, $G_n = a\bigl(\frac{1}{n}\sum_i \epsilon_i x_i^\top\bigr)$ is a rank-one matrix of the form $a v^\top$ with $v = \frac{1}{n}\sum_i \epsilon_i x_i$, so
$\fnorm{G_n} = \fnorm{a v^\top} = \|a\|_2\,\|v\|_2 = \|v\|_2$ (using $\|a\|_2=1$).

\textbf{Step 3: bound $\E\|v\|_2$ via Jensen and independence.}
By Jensen's inequality,
\begin{equation}
  \E_\epsilon \|v\|_2 \;\le\; \sqrt{\E_\epsilon \|v\|_2^{\,2}}
  \;=\; \sqrt{\E_\epsilon \frac{1}{n^2}\sum_{i,j}\epsilon_i \epsilon_j\, x_i^\top x_j}
  \;=\; \sqrt{\frac{1}{n^2}\sum_{i=1}^{n} \|x_i\|_2^{\,2}}
  \;\le\; \frac{R}{\sqrt{n}},
\end{equation}
where the cross-terms vanish because $\E[\epsilon_i\epsilon_j]=\delta_{ij}$.

\textbf{Step 4: combine.}
Taking expectation over $\epsilon$,
\begin{equation}
  \hat{\mathfrak{R}}_n(\mathcal{H}_a(\beta))
  \;=\; \beta\, \E_\epsilon \fnorm{G_n}
  \;\leq\; \beta\,\frac{R}{\sqrt{n}},
\end{equation}
which is~\eqref{eq:rademacher}.
\end{proof}

\paragraph{Remark on rank.}
The bound makes no use of $\rank(\dW)\le r$, so it applies uniformly to any Frobenius-bounded linear adapter, low-rank or not. A rank constraint would buy a sharper bound only if it were combined with a nuclear-norm constraint $\|\dW\|_*\leq\sqrt{r}\beta$ via standard low-rank Rademacher arguments; we do not pursue this since for our purposes the rank-free bound already suffices to show that grouping does not enlarge the hypothesis class.

\section{Additional Experimental Details}
\label{app:exp_details}

\begin{table}[t]
  \centering
\caption{%
  Sensitivity to the number of groups $\nK$ on Flan-T5-Base and Flan-T5-Large
  / GLUE; each cell reports Base\,/\,Large. Partitions come from the same clustering GRAM uses; every group is re-merged with the
  adaptive-rank merge ($r{=}16$) and freshly evaluated with the
  \Cref{tab:glue_compact} pipeline, which reproduces the published $\nK{=}2$
  row exactly on Large (on Base, re-merged $\nK{=}1$ scores 81.9 vs.\ 81.7 for
  the original checkpoint---SVD tie-breaking on near-degenerate spectra). The
  Weight Averaging row and the Large $\nK{\in}\{1,8\}$ anchors are the
  corresponding \Cref{tab:glue_compact} rows.
}
\label{tab:k_sweep}
\scriptsize
\setlength{\tabcolsep}{3pt}
\renewcommand{\arraystretch}{1.0}
\resizebox{\textwidth}{!}{
\begin{tabular}{l c c c c c c c c c r}
\toprule
$\nK$ & CoLA & MNLI & MRPC & QNLI & QQP & RTE & SST-2 & STS-B & Avg. & Extra params \\
\midrule
1 (weight averaging) & 69.7 / 74.6 & 59.7 / 84.3 & 78.9 / 84.1 & 90.1 / 92.8 & 83.8 / 86.2 & 80.5 / 85.9 & 91.2 / 94.8 & 72.0 / 88.0 & 78.2 / 86.3 & --- \\
1 (global merge)     & 69.1 / 79.2 & 74.2 / 86.0 & 80.4 / 87.3 & 89.7 / 93.6 & 83.5 / 85.6 & 80.9 / 87.0 & 92.7 / 95.3 & 84.6 / 88.7 & 81.9 / 87.8 & --- \\
\textbf{2 (GRAM)}    & 69.5 / 79.3 & 82.8 / 89.0 & 84.3 / 88.7 & 90.2 / 93.8 & 84.0 / 85.7 & 79.8 / 88.1 & 93.2 / 95.4 & 86.7 / 89.9 & \textbf{83.8} / \textbf{88.7} & $+1.8$M / $+4.7$M \\
3                    & 69.5 / 79.3 & 82.7 / 88.5 & 84.3 / 88.7 & 90.2 / 93.8 & 84.0 / 85.7 & 79.8 / 88.1 & 93.2 / 95.4 & 87.4 / 90.9 & 83.9 / 88.8 & $+3.5$M / $+9.4$M \\
4                    & 69.1 / 79.9 & 82.7 / 88.5 & 84.8 / 89.5 & 90.2 / 93.8 & 83.7 / 85.8 & 78.0 / 88.4 & 93.2 / 95.2 & 87.4 / 90.9 & 83.6 / 89.0 & $+5.3$M / $+14.2$M \\
5                    & 69.1 / 80.1 & 82.7 / 88.5 & 84.8 / 89.5 & 90.2 / 93.8 & 83.7 / 85.8 & 78.0 / 88.4 & 92.9 / 95.2 & 87.4 / 90.9 & 83.6 / 89.0 & $+7.1$M / $+18.9$M \\
6                    & 69.1 / 80.1 & 82.7 / 88.5 & 83.3 / 89.5 & 90.9 / 93.9 & 83.5 / 85.8 & 79.8 / 90.6 & 92.9 / 95.2 & 87.4 / 90.9 & 83.7 / 89.3 & $+8.8$M / $+23.6$M \\
7                    & 69.1 / 80.1 & 82.7 / 88.5 & 83.1 / 88.7 & 90.9 / 94.4 & 84.1 / 87.1 & 79.1 / 88.8 & 92.9 / 95.2 & 87.4 / 90.9 & 83.7 / 89.2 & $+10.6$M / $+28.3$M \\
8 (no merging)       & 69.1 / 80.1 & 82.7 / 88.5 & 85.5 / 89.2 & 90.9 / 94.4 & 84.1 / 87.1 & 84.5 / 91.3 & 92.9 / 95.2 & 87.4 / 90.9 & 84.6 / 89.6 & $+12.4$M / $+33.0$M \\
\bottomrule
\end{tabular}
}
\end{table}

\subsection{Grouping Analysis}
\label{app:grouping_analysis}

\begin{table}[t]
\centering
\caption{%
  GRAM vs.\ random grouping quality. Intra-sim and Inter-sim denote the average
  subspace similarity ($1 - d_c^2/r$) within and between groups, respectively.
  Ratio = Intra-sim / Inter-sim. Higher ratio indicates better separation.
}
\label{tab:grouping_analysis}
\setlength{\tabcolsep}{6pt}
\begin{tabular}{l c c c c}
\toprule
Setting & Intra-sim & Inter-sim & Ratio & Avg Acc \\
\midrule
No grouping ($\nK{=}1$, global merge) & 0.097 & --- & --- & 81.7 \\
\midrule
GRAM $\nK{=}2$ & 0.143 & 0.035 & 4.06 & 83.8 \\
Random $\nK{=}2$   & 0.089 & 0.102 & 0.87 & 82.4 \\
\midrule
GRAM $\nK{=}3$ & 0.150 & 0.036 & 4.21 & 83.9 \\
Random $\nK{=}3$   & 0.119 & 0.089 & 1.34 & 83.1 \\
\bottomrule
\end{tabular}
\end{table}

\Cref{tab:grouping_analysis} quantifies the grouping quality produced by GRAM
versus random assignment. GRAM with $\nK{=}2$ achieves an intra-group to inter-group
similarity ratio of 4.06, compared to 0.87 for random grouping---a factor-of-four difference. This
separation increases slightly to 4.21 at $\nK{=}3$. Across the four grouped settings, higher ratio correlates with better downstream accuracy: geometric coherence within groups translates directly to merge quality.

Random grouping at $\nK{=}2$ produces groups whose intra-group similarity (0.089) is actually
\emph{lower} than the global average (0.097), because some random partitions place highly
dissimilar tasks together. GRAM avoids this failure mode by construction.

\subsection{Deployment without Oracle Task Identity}
\label{app:router}

GRAM targets task-aware deployments---e.g., a multi-institution medical system where each request carries its provenance, so task identity is deployment \emph{metadata} rather than an inference problem. Nevertheless, GRAM extends to the task-agnostic setting at no accuracy cost via a lightweight router. We instantiate the Mahalanobis task mapper of \citet{lee2018unified}, as deployed for task recognition in shared-knowledge lifelong learning (SKILL) \citep{ge2023lightweight}: class-conditional Gaussians in the \emph{frozen base encoder's} feature space (mean-pooled hidden states; no extra model), class means estimated from only $m{=}5$ training samples per class ($\approx 85$ raw examples total across all 8 tasks), and one tied covariance (Ledoit--Wolf). A test input is routed to the group owning its nearest class centroid. The router involves no gradient training and no test-time task labels; merging itself remains fully data-free, and the router's $\approx 85$-example requirement is strictly weaker than the $64$-samples-\emph{per-task} calibration consumed by the data-using baselines (RegMean, Fisher-Merge) in \cref{sec:qwen_medical}. On Flan-T5-Base / GLUE at $\nK{=}2$, the router routes $100.0\%$ of validation examples of every task to the correct group, so task-agnostic GRAM \emph{equals} oracle GRAM---$83.8$ on every task, to the decimal.

\begin{table}[h]
\centering
\caption{%
  Group-routing accuracy (\%) of the Mahalanobis task mapper under the
  partition chain for every $\nK$, Flan-T5-Base / GLUE. $\nK{=}8$ requires
  full task identification.
}
\label{tab:router_ksweep}
\setlength{\tabcolsep}{6pt}
\begin{tabular}{l c c c c c c c}
\toprule
$\nK$ & 2 & 3 & 4 & 5 & 6 & 7 & 8 (full task ID) \\
\midrule
Group-routing acc. & 100.00 & 100.00 & 100.00 & 100.00 & 99.99 & 99.98 & 99.98 \\
\bottomrule
\end{tabular}
\end{table}

Small $\nK$ is what makes routing easy. The router encodes the same formatted request the instruction-tuned model itself consumes, and it only needs to resolve \emph{groups}, not tasks: at $\nK{=}2$ it faces a binary decision, and any confusion between tasks in the same coalition is harmless by construction. \Cref{tab:router_ksweep} reports the mapper's group-routing accuracy under the partition chain for every $\nK$. The mapper's \emph{only} errors, at any $\nK$, are QQP$\leftrightarrow$MRPC confusions ($72$ of $40{,}430$ QQP examples, $0.18\%$)---two semantically adjacent paraphrase-detection tasks. These errors are invisible for $\nK \le 5$ because GRAM's geometry-aware clustering places QQP and MRPC in the \emph{same} coalition, and only surface once $\nK \ge 6$ separates them. This reveals a structural synergy: a router's confusions occur precisely between similar tasks, and similar tasks are precisely what GRAM groups together---the coalition structure absorbs router errors by design. A per-task deployment ($\nK{=}8$) forgoes this protection and demands full task identification, exactly the harder regime studied in SKILL ($87.1\%$ over 102 tasks) \citep{ge2023lightweight}.

\subsection{Per-Task Intrinsic Rank Heterogeneity}
\label{app:rank_heterogeneity}

\Cref{fig:singular_value_count} reports the number of singular values needed to capture $\ge 99\%$ of the per-layer Frobenius energy, computed over the 8 GLUE tasks and 72 LoRA layers of Flan-T5-Base. Uniform per-task allocation $r/|G_k|$ inside a coalition either wastes slots on compact-spectrum tasks or truncates long-tail ones; the adaptive allocator of \cref{sec:group_merge} reads off this heterogeneity directly from the singular values and assigns rank where it actually reduces error.

\begin{figure}[h]
\centering
\includegraphics[width=\linewidth]{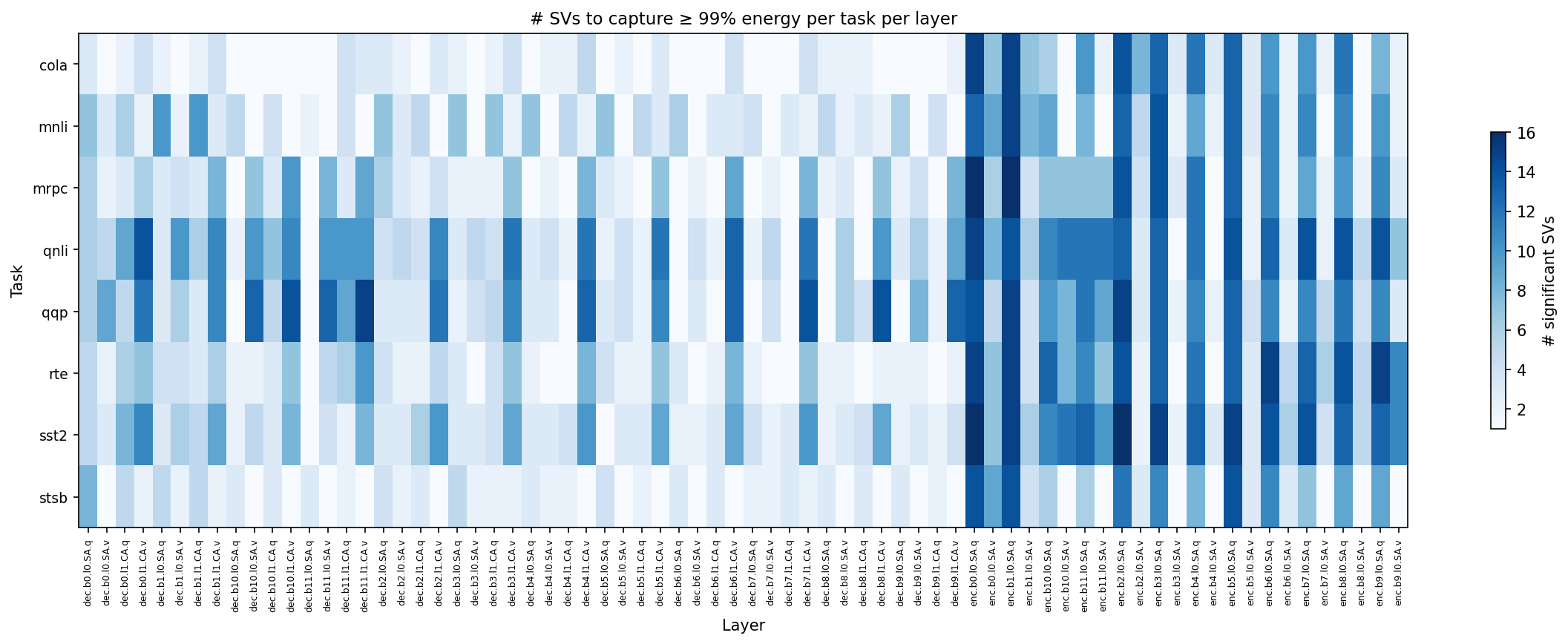}
\caption{Number of singular values required to capture $\ge 99\%$ of Frobenius energy, per task (rows) and per LoRA layer (columns), on Flan-T5-Base. Decoder layers (left half) are typically low-rank ($1$--$4$ significant directions); encoder layers (right half) are higher-rank ($10$--$16$). The cross-task and cross-layer heterogeneity motivates the adaptive rank allocator in \cref{sec:group_merge}.}
\label{fig:singular_value_count}
\end{figure}

\subsection{Pairwise Projection Error on ViT-B/32}
\label{app:vit_proj_error}

\Cref{fig:sigma_chordal_k1_vs_k2} visualizes the pairwise projection error $\sum_s\sigma_{i,s}^2\sin^2\theta_s$ between task subspaces on ViT-B/32 under $\nK{=}1$ (TSV-Merge) vs.\ $\nK{=}2$ (GRAM); the block structure that emerges at $\nK{=}2$ is described in the caption.

\begin{figure}[h]
\centering
\includegraphics[width=\linewidth]{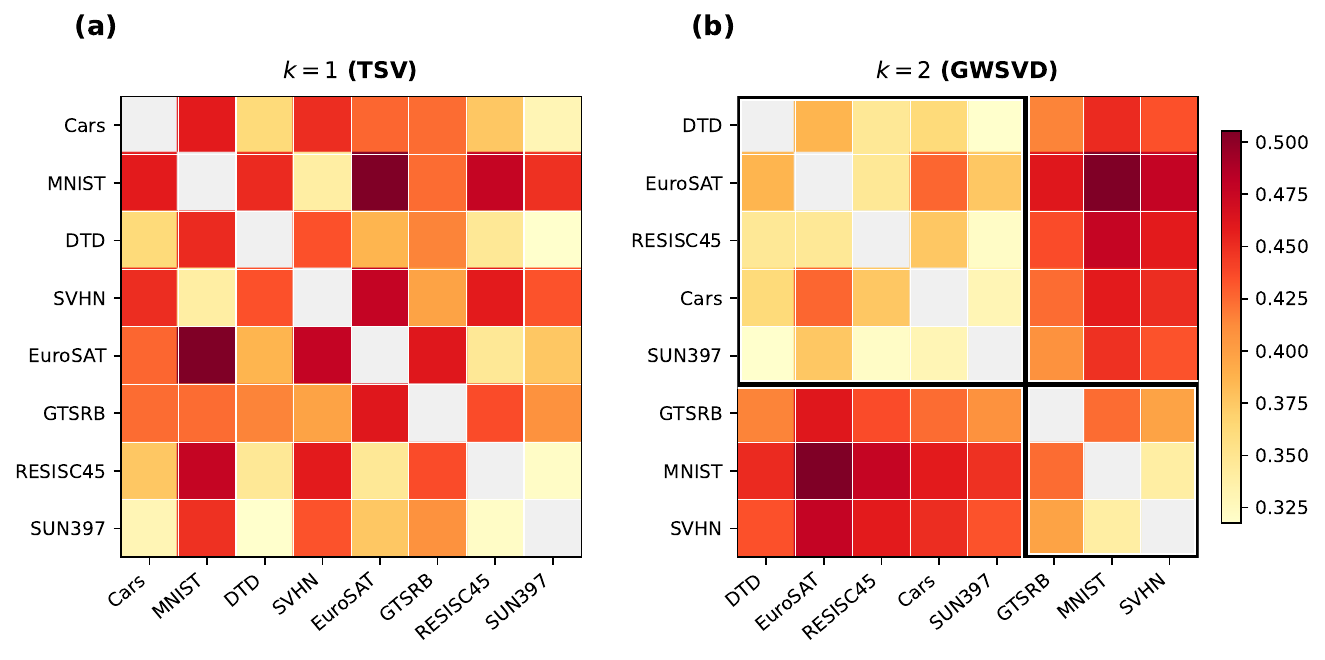}
\caption{Pairwise projection error $\sum_s \sigma_{i,s}^2\sin^2\theta_s$ between task subspaces on ViT-B/32, under $\nK{=}1$ (TSV-Merge, left) and $\nK{=}2$ (GRAM, right). Reordering at $\nK{=}2$ by GRAM's partition exposes a block-diagonal structure: within-group pairs incur lower error than across-group pairs, consistent with \cref{lem:alignment_penalty}.}
\label{fig:sigma_chordal_k1_vs_k2}
\end{figure}

\subsection{CLIP ViT-B/16 Full Fine-tuning Results}
\label{app:vit_b16_ft}

\begin{table}[t]

  \centering
\caption{%
  Eight-task vision benchmark results (test accuracy)
  using CLIP ViT-B/16 with full fine-tuning.
  \textbf{Bold}: best merged result per column;
  \underline{underline}: second best.
  Individual adapters use one task-specific model per task (oracle upper bound).
}
\label{tab:clip_vit_b16_ft}
\resizebox{\textwidth}{!}{
\begin{tabular}{l r r r r r r r r r}
\toprule
Method & MNIST & SVHN & GTSRB & DTD & Cars & RESISC45 & EuroSAT & SUN397 & Avg. \\
\midrule
Individual (oracle) & 99.8 & 97.8 & 99.0 & 98.3 & 87.5 & 97.0 & 99.8 & 78.8 & 94.8 \\
\midrule
Weight Averaging     & 94.4 & 72.6 & 60.6 & 54.0 & 69.9 & 75.2 & 81.3 & 68.6 & 72.1 \\
Task Arithmetic      & 98.3 & 84.5 & 76.1 & 60.4 & 68.1 & 73.7 & 86.4 & 64.4 & 76.5 \\
TIES-Merging         & 98.3 & 84.8 & 73.9 & 64.8 & 71.7 & 77.8 & 89.4 & 68.0 & 78.6 \\
AdaMerging++         & 98.4 & 93.1 & 92.3 & 64.0 & 79.2 & 81.1 & 95.6 & 69.2 & 84.1 \\
TSV-Merge            & \underline{99.3} & 92.9 & 94.5 & 85.2 & 81.0 & 88.5 & 97.4 & 70.6 & 88.7 \\
Iso-C                & 99.1 & 90.7 & 94.9 & 89.9 & 82.6 & \underline{91.6} & 97.4 & \underline{73.5} & 90.0 \\
Iso-CTS              & 99.2 & 90.4 & \underline{96.2} & \textbf{92.3} & \textbf{83.2} & \textbf{92.8} & \underline{97.7} & \textbf{74.1} & \underline{90.7} \\
LiNeS-Merge          & 98.6 & 85.6 & 79.7 & 64.7 & 72.5 & 77.5 & 87.6 & 65.4 & 79.0 \\
LOT-Merging          & 99.0 & \underline{94.8} & 70.8 & 68.7 & 75.9 & 82.6 & 84.4 & 70.7 & 80.9 \\
\midrule
\textbf{GRAM} (ours)          & \textbf{99.5} & \textbf{96.0} & \textbf{96.8} & \underline{92.2} & \underline{82.7} & 91.4 & \textbf{98.8} & 72.0 & \textbf{91.2} \\
\bottomrule
\end{tabular}
}
\end{table}

For completeness across the standard CLIP backbone family, \cref{tab:clip_vit_b16_ft} reports per-task accuracy on the 8-task vision benchmark using a fully fine-tuned CLIP ViT-B/16 checkpoint. GRAM reaches $91.2\%$ average---the highest among all compared methods---outperforming Iso-CTS ($90.7\%$), Iso-C ($90.0\%$), TSV-Merge ($88.7\%$), and the remaining single-model baselines. The win pattern mirrors the B/32 and L/14 results in \cref{tab:clip_vit_ft_compact}: the largest absolute gains over global merging concentrate on the high-chordal-distance tasks (GTSRB, EuroSAT, SVHN), exactly the cells that the alignment-penalty analysis of \cref{lem:alignment_penalty} predicts will benefit most from grouping. The consistency of the gain pattern across three ViT scales (B/32, B/16, L/14) is consistent with the partition-quality mechanism being scale-invariant rather than tied to a specific backbone.

\subsection{Implementation Details}
\label{app:impl_details}

\paragraph{Flan-T5 GLUE adapters.}
On both Flan-T5-Base and Flan-T5-Large we use LoRA with rank $r = 16$ and $\alpha = 32$ (scaling factor $2.0$), applied to the query and value projections in self-attention and (in the decoder) cross-attention.  Each LoRA delta has shape $768 \times 768$ on Base and $1024 \times 1024$ on Large; the number of adapted layers is $\nL = 72$ on Base (12 encoder + 12 decoder blocks) and $\nL = 144$ on Large (24 encoder + 24 decoder blocks).  Per-task adapters are trained with AdamW (learning rate $5\times 10^{-4}$, batch size $16$) for $10$ epochs; we report on the GLUE validation splits, with Spearman $\rho{\times}100$ on STS-B and accuracy on the remaining tasks.

\paragraph{CLIP ViT vision adapters.}
We use the per-task fully fine-tuned CLIP ViT-B/32 and ViT-L/14 checkpoints released by TSV-Merge \citep{gargiulo2025tsv}, evaluated on the standard 8-task vision benchmark (MNIST, SVHN, GTSRB, DTD, Cars, RESISC45, EuroSAT, SUN397).  The ViT-B/16 results in \cref{app:vit_b16_ft} are produced under the same protocol on a fully fine-tuned ViT-B/16 backbone.  Since the per-task checkpoints are fully fine-tuned, the per-layer task delta $\dW_i^{(\ell)} = W_i^{(\ell)} - W_0^{(\ell)}$ is full-rank; we apply SVD truncation at merge time with rank budget $r{=}16$ per layer (matching the GLUE setting), which is the rank actually flowing into the chordal distance and the per-coalition merge.  As stated in \cref{sec:preliminaries}, the matrix formulation applies to linear layers: SVD-based merging is applied to all 2-D weight matrices (MLP blocks and attention projections), while the remaining non-matrix parameters (LayerNorm gains, biases) are merged by simple averaging.  All vision results report test accuracy.  No additional task-specific training is performed; GRAM and all baselines operate on these checkpoints directly.

\paragraph{Qwen2 medical adapters.}
We use four publicly released LoRA adapters on Qwen2-7B-Instruct, one per medical sub-domain: MedQA, CMExam, HealthcareMagic, and cMedQA2.  Each adapter uses LoRA with rank $r = 8$ and $\alpha = 16$, applied to the attention ($q,k,v,o$) and MLP (gate, up, down) projections of every decoder block.  GRAM operates directly on these adapters without further training.

\paragraph{Calibration set for data-using baselines.}
Fisher-Merge \citep{matena2022merging} and RegMean \citep{jin2022dataless} require a small calibration set per task at merge time.  We use a uniform $64$-sample-per-task calibration set drawn from each task's training split: Fisher-Merge uses these samples to estimate per-parameter Fisher information, and RegMean uses them to estimate input second-moment matrices for the closed-form layer-wise solver.  The data-free baselines (Weight Averaging, Task Arithmetic, TIES-Merging, Iso-C, Iso-CTS, TSV-Merge) and GRAM see no calibration data.

\paragraph{Clustering.}
For each model family we compute the $\nT \times \nT$ pairwise chordal distance matrix (averaged across all $\nL$ adapted layers) and run $k$-means++ with $20$ random initializations on the projection-matrix embeddings.  We set $\nK = 2$ as the number of groups throughout.

\paragraph{Computational cost.}
For the Flan-T5-Base/GLUE setup ($\nT{=}8$, $\nL{=}72$), SVD decomposition of all $\nT \times \nL$ task deltas takes approximately $20$ seconds on CPU; distance-matrix construction and $k$-means++ clustering complete in under $1$ second; per-group SVD merging adds negligible overhead.  Total wall-clock time for the GRAM pipeline is under $30$ seconds, compared to hours of adapter training per task.  The other model families exhibit the same qualitative cost structure: SVD dominates and is bounded by the per-layer cost $O(\dout\din r)$ summed over $\nT\nL$ deltas.

\subsection{Medical Benchmark Details}
\label{app:medical_details}

\paragraph{Task suite.}
The medical benchmark covers two task families crossed with two languages:
\begin{itemize}[nosep,leftmargin=*]
  \item \emph{Multiple-choice QA (MCQA).}  MedQA (English) and CMExam (Chinese): the model selects one option from a candidate set; reported metric is exact-match accuracy.
  \item \emph{Doctor--patient dialogue (open-ended generation).}  HealthcareMagic (English) and cMedQA2 (Chinese): the model generates a free-form answer to a patient query; reported metrics are BLEU-4 and ROUGE-1, ROUGE-2, and ROUGE-L (F1) against the reference response.
\end{itemize}
The 4 task families with 1 metric each (MCQA accuracy) plus 4 metrics each (BLEU-4, ROUGE-1/2/L) on the two open-ended tasks yield the $10$ heterogeneous metric columns reported in \cref{tab:qwen_medical_main}.

\paragraph{Summary metrics.}
\emph{Avg Rank} is the mean rank across the $10$ columns, where rank $1$ is the best per-column merging method (the per-task Individual oracle is excluded from ranking).  \emph{Avg Impr} is the per-column relative improvement over the un-adapted Pretrained baseline, averaged across the $10$ columns: for each metric column $j$ and method $M$, we compute $(\mathrm{score}_{M,j} - \mathrm{score}_{\mathrm{Pretrained},j})/\mathrm{score}_{\mathrm{Pretrained},j}$, then average over $j$.  This normalizes for the wide range of metric magnitudes (MCQA accuracy in tens of points; ROUGE-2 in single digits).

\end{document}